\documentclass{article}

\PassOptionsToPackage{numbers,sort&compress}{natbib}

\usepackage[preprint]{neurips_2026}

\usepackage[utf8]{inputenc} 
\usepackage[T1]{fontenc}    
\usepackage{hyperref}       
\usepackage{url}            
\usepackage{booktabs}       
\usepackage{amsfonts}       
\usepackage{nicefrac}       
\usepackage{microtype}      
\usepackage{xcolor}         
\usepackage{amssymb}

\usepackage{amsthm} 
\usepackage{amsmath}

\AtBeginDocument{%
  \setlength{\abovedisplayskip}{0.6ex}
  \setlength{\belowdisplayskip}{0.6ex}
}
\allowdisplaybreaks[4]

\theoremstyle{definition} 
\usepackage{amsthm}

\newtheorem{proposition}{Proposition}
\newtheorem{assumption}{Assumption}
\newtheorem{corollary}{Corollary} 

\usepackage[utf8]{inputenc} 
\usepackage[T1]{fontenc}    
\usepackage{hyperref}       
\usepackage{ifsym}
\usepackage{url}            
\usepackage{multirow}
\usepackage{enumitem}
\setlist[itemize]{leftmargin=*}
\usepackage{booktabs}       
\usepackage{amsfonts}       
\usepackage{nicefrac}       
\usepackage{microtype}      
\usepackage{xcolor}         
\usepackage{graphicx}
\usepackage{subcaption}

\usepackage[ruled,vlined,linesnumbered]{algorithm2e}

\title{Multimodal Continuous Reasoning via Asymmetric Mutual Variational Learning}

\author{%
  Shijie Li$^{1,2}$\thanks{Equal contribution.} \quad
  Yilin Gao$^{2}$\footnotemark[1] \quad
  Siyuan Yang$^{2}$ \quad
  Tieyuan Chen$^{1}$ \quad
  Chaofan Gan$^{1}$ \\
  \textbf{Zhihao He}$^{1}$ \quad
  \textbf{Zicheng Zhao}$^{1}$ \quad
  \textbf{Yuyu Guo}$^{2}$\thanks{Corresponding author.} \quad
  \textbf{Weiyao Lin}$^{1}$\footnotemark[2] \quad
  \textbf{Hang Yu}$^{2}$\footnotemark[2] \\
  \\
  $^1$ Shanghai Jiao Tong University \\
  $^2$ Ant Group \\
  \texttt{\{shijieli, wylin\}@sjtu.edu.cn} \\
  \texttt{\{fhlyhv, yuyuguo1994\}@gmail.com} \\
}
\begin{document}

\maketitle

\begin{abstract}
Multimodal Large Language Models (MLLMs) are often constrained by a language-space bottleneck, forcing complex visual reasoning into discrete tokens which can lose perceptual nuance. A promising alternative is continuous latent reasoning, where the goal is to discover implicit reasoning pathways that bridge the multimodal query and the final answer. However, this introduces a severe train-inference mismatch: a training-time posterior, conditioned on the ground-truth answer, can exploit answer-dependent shortcuts. Standard variational training then forces the inference-time prior to mimic a posterior that has access to information unavailable at test time, leading to poor performance. To address this, we propose Asymmetric Mutual Variational Learning (AMVL), a framework that resolves this mismatch via a bidirectional calibration objective. A forward KL divergence trains the target-agnostic prior to match the posterior, while a novel reverse KL divergence simultaneously regularizes the posterior, preventing it from collapsing into inference-incompatible regions and mitigating this ``answer leakage''. We provide theoretical analysis formalizing this leakage as prior contamination and prove that our dual-KL objective reduces it. We instantiate AMVL in a latent-integrated MLLM and show that it consistently outperforms strong discrete and latent-reasoning baselines, improving the average score on the complex BLINK benchmark by +10.83 and achieving gains of up to +32.00 on individual reasoning tasks, with analyses confirming improved latent-space stability.
\end{abstract}

\section{Introduction}
Human reasoning is inherently multimodal. When we solve a visual puzzle or interpret a complex diagram, we think directly in perceptual, spatial, and abstract representations that language alone cannot fully capture~\cite{lecun2022path, mahowald2024dissociating}. This poses a fundamental challenge for Multimodal Large Language Models (MLLMs): if the intermediate reasoning process is constrained to the discrete token space of natural language, the model is forced to verbalize visual concepts that are intrinsically continuous and high-dimensional. The result is a systematic 
\textbf{language-space bottleneck}—reasoning quality is limited not by the model's representational capacity, but by the expressive constraints of the discrete language space through which all intermediate thought must pass. This bottleneck is particularly acute in vision-language tasks demanding fine-grained spatial abstraction or multi-step planning, where text-based Chain-of-Thought (CoT) can cause models to drift from the visual input, introduce hallucinations, and lose precise perceptual grounding~\cite{xu2026thinking, guan2024hallusionbench}.


These limitations have motivated a growing body of work on latent visual reasoning, where models perform intermediate steps directly in a continuous embedding space~\cite{lvr, pham2025multimodal, yang2025machine}. Recent methods, including LVR~\cite{lvr}, Monet~\cite{monet}, and Mull-Tokens~\cite{mull}, have shown promise by replacing discrete reasoning tokens with continuous latent states. However, these pioneering methods share a \textbf{critical and underexplored limitation}: they all rely on \textbf{explicit, hand-crafted supervision signals}—such as reconstruction objectives or alignment losses—to shape what these latent states should encode. The latent reasoning process is thus constrained to encode whatever the designer has pre-specified as important, rather than being free to discover the representations most useful for bridging the input question and the final answer.


We argue that a more principled alternative lies in formulating latent reasoning as a structured probabilistic inference problem~\cite{kingma2013auto, sohn2015learning}. Instead of prescribing what latent states should encode, this approach allows the model to discover the intermediate representations that most naturally bridge the multimodal input and the target output. This framing immediately suggests a target-aware posterior over latent states for training and a target-agnostic prior for inference. However, applying this idea to powerful autoregressive MLLMs introduces a severe train-inference mismatch~\cite{lin2025ravr,wangvariational}. When the posterior has access to the reference answer, it can rely on answer-dependent shortcuts that are unavailable at test time. Under the standard evidence lower bound (ELBO)~\cite{kingma2013auto}, a forward KL term then trains the prior to imitate these posterior latents, causing the inference-time prior to inherit a latent geometry partially shaped by information leakage. As a result, the prior is poorly calibrated for test-time reasoning even though it is optimized to approximate the posterior during training.

To address this, we propose \textbf{Asymmetric Mutual Variational Learning (AMVL)}, a unified framework for multimodal continuous reasoning. AMVL directly tackles the train-inference mismatch by establishing an asymmetric mutual learning process between the prior and the posterior. It relies on two complementary regularizers: a \textbf{forward KL alignment} term that aligns the prior with the posterior-inferred latent states, and a \textbf{reverse KL regularization} term that constrains the posterior relative to the learned prior. Crucially, these are not symmetric peer-learning losses in the style of deep mutual learning; rather, they regulate two distributions with fundamentally different conditioning structures and downstream roles. This dual-KL calibration closes the latent gap from both sides, replacing hand-crafted latent supervision with a self-contained, answer-driven signal. It encourages the latent space to be both expressive during training and well-calibrated for inference. We instantiate AMVL in a latent-integrated MLLM architecture with lightweight variational heads, making it efficient and seamlessly applicable.

In summary, our contributions are as follows:
\begin{itemize}[leftmargin=*, itemsep=0.2pt, topsep=0pt, partopsep=0pt]
    \item We identify train-inference mismatch between a target-aware posterior and a target-agnostic prior as the central obstacle in variational multimodal latent reasoning, and clarify why this problem is not adequately addressed by standard one-sided ELBO training.
    \item We propose AMVL, an asymmetric mutual learning framework that combines forward prior alignment with reverse posterior support regularization, enabling end-to-end discovery of a continuous reasoning space without external latent supervision.
    \item We demonstrate through extensive experiments that AMVL consistently outperforms strong discrete and latent-reasoning baselines on challenging multimodal benchmarks, confirming the benefits of inference-compatible continuous latent reasoning.
\end{itemize}

\section{Related Work}



\subsection{Discrete Multimodal Reasoning}

Recent advancements optimize Multimodal Large Language Models (MLLMs) to explicitly \textit{``think about images''}~\cite{hong2025apo, jiang2025vlmr3regionrecognitionreasoning, zhang2023multimodal, ni2025pointrftimprovingmultimodalreasoning}. Frameworks like Vision-R1~\cite{visionr1} and PAPO~\cite{papo} leverage reinforcement learning to elicit extensive natural language Chain-of-Thought (CoT). While interpretable, this paradigm suffers directly from the \textbf{language-space bottleneck}: forcing continuous, high-dimensional visual signals into a rigid discrete vocabulary inevitably loses fine-grained perceptual nuances, often causing reasoning drift and hallucinations.
To mitigate this, a concurrent \textit{``thinking with images''} paradigm~\cite{su2025thinking, sarch2025grounded, chen2025mint, chung2025don, zhang2025chain} intertwines pixel-level visual features directly into intermediate reasoning steps. However, while models like PixelReasoner~\cite{pixelreasoner} and DeepEyes~\cite{deepeyes} improve spatial grounding, a fundamental structural limitation remains: their overarching reasoning trajectories are still strictly bound to discrete, autoregressive text generation.

Unlike discrete paradigms, AMVL bypasses textual discretization by operating entirely in a continuous latent space. Leveraging the \textbf{inherently high information density} of continuous vectors, AMVL allows abstract spatial logic to fluidly evolve, eliminating discrete vocabulary constraints and ensuring rich perceptual details are preserved.

\subsection{Continuous Latent Reasoning}

To move beyond purely discrete reasoning trajectories, a growing body of work introduces latent tokens as internal computation slots within autoregressive models~\cite{butt2025soft, geiping2025scaling, hao2022training, shen2025codi, wang2025synadapt}. Early explorations, such as Pause Tokens~\cite{goyal2023think}, demonstrated that appending dummy tokens before generating the answer allows language models to perform implicit computation. More recently, representative continuous reasoning methods such as Mull-Tokens~\cite{mull}, LVR~\cite{lvr}, Monet~\cite{monet}, and Coconut~\cite{coconut} extend the input sequence with continuous hidden-state capacity, enabling the model to perform intermediate computation without emitting every step as discrete text. These methods suggest that continuous latent reasoning can provide additional expressiveness and computational flexibility.

While existing approaches rely on explicit supervision—reducing latent states to mere compressed proxies—we reformulate the reasoning trajectory as an unobserved stochastic variable. Through variational modeling, our continuous space is optimized solely for effective multimodal reasoning, rather than superficial signal reconstruction.

\subsection{Variational Inference for Latent Reasoning}

Variational inference offers a principled framework for conditional generation~\cite{kingma2013auto, sohn2015learning}, though applying it to autoregressive decoders historically requires mitigations against posterior collapse~\cite{iyyer2015deep, he2019lagging, higgins2017beta}. Recently, this perspective has formalized LLM reasoning trajectories as latent variables~\cite{hu2023amortizing, zhou2025variational}. However, whether compressing visual-semantic CoT (ReGuLaR~\cite{wangvariational}) or steering discrete CoT, these methods bind their optimization to the \textbf{explicit alignment of step-by-step traces}. Approaches like RAVR~\cite{lin2025ravr} use reference-guided posteriors that construct trajectories, making them highly susceptible to \textbf{answer leakage}. The posterior exploits the answer as an \textbf{informational shortcut}, forcing the target-agnostic prior to mimic this hindsight bias and causing a severe \textbf{train-inference mismatch}.

Our approach shares a conceptual resemblance with deep mutual learning~\cite{deepmutual}, as both frameworks involve two distributions mutually regularizing each other during training. However, while traditional mutual learning employs symmetric peers collaborating over the same output space, our prior and posterior are \textbf{fundamentally asymmetric}. They possess distinct conditioning structures (target-agnostic vs. target-aware) and operational phases (inference vs. training). By jointly optimizing bidirectional KL objectives, AMVL achieves an \textbf{asymmetric mutual calibration}: the forward KL aligns the prior with posterior-inferred states, while the reverse KL restricts the posterior from drifting into inference-incompatible regions, effectively mitigating answer leakage.

\begin{figure}[t]
  \vspace{-10pt}
  \centering
  \includegraphics[width=0.9\columnwidth]{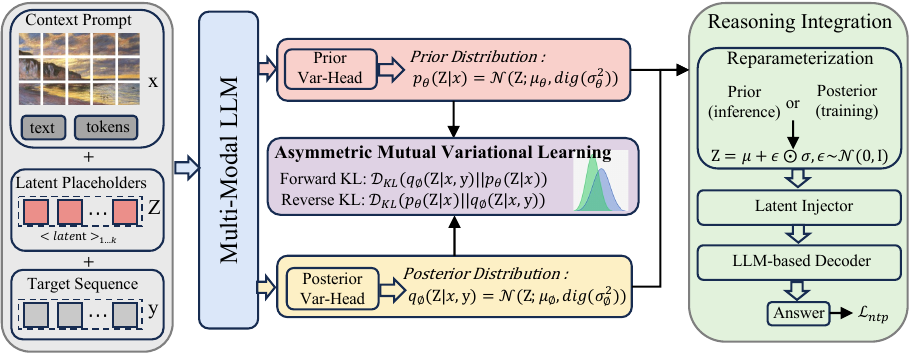}
\caption{
Overview of AMVL for multimodal continuous reasoning. Given a multimodal prompt, the model inserts $k$ latent slots into the autoregressive sequence and uses their hidden states to parameterize a target-agnostic prior $p_\theta(Z|x)$ and a target-aware posterior $q_\phi(Z|x,y)$ over latent sequences $\mathbf{Z}$. During training, posterior samples are injected into the latent slots for decoding, while forward and reverse KL terms jointly calibrate the prior and regularize the posterior. During inference, latent features are sampled from the prior and used to guide autoregressive answer generation.
}
  \label{overview}
  \vspace{-10pt}
\end{figure}

\section{Method}

We present Asymmetric Mutual Variational Learning (AMVL), a unified framework for multimodal continuous reasoning that jointly learns a continuous latent space and an autoregressive decoder. As illustrated in Figure~\ref{overview}, AMVL frames latent reasoning as a probabilistic inference problem, allowing the model to \emph{discover} intermediate representations that best connect a multimodal input to its target answer, rather than relying on hand-crafted supervision. Our core contribution is a bidirectional latent calibration objective designed to resolve the train-inference mismatch inherent in this setup.

\subsection{A Variational Approach to Latent Reasoning}

Our approach is founded on the idea that latent reasoning variables can be learned by treating them as unobserved components within a conditional generative model.

\textbf{Discovering Reasoning via Latent-Variable Modeling.}
We introduce a sequence of continuous latent variables
$\mathbf{Z} = [\mathbf{z}_1, \ldots, \mathbf{z}_k] \in \mathbb{R}^{k \times d}$
to represent intermediate reasoning steps. The conditional log-likelihood of generating an answer $\mathbf{y}$ from a multimodal context $\mathbf{x}$ is then expressed by marginalizing over these latent variables:
\begin{equation}
\log p_\theta(\mathbf{y} \mid \mathbf{x})
=
\log \int p_\theta(\mathbf{y} \mid \mathbf{x}, \mathbf{Z})\, p_\theta(\mathbf{Z} \mid \mathbf{x})\, d\mathbf{Z}.
\end{equation}
Here, $p_\theta(\mathbf{Z} \mid \mathbf{x})$ is a \textbf{target-agnostic prior} distribution over reasoning states used for inference, and $p_\theta(\mathbf{y} \mid \mathbf{x}, \mathbf{Z})$ is the decoder that generates the answer conditioned on the latent sequence. Because the integral in Eq.~1 is intractable, we turn to variational inference~\cite{kingma2013auto,sohn2015learning}. This provides a principled way to learn the latent variables by introducing a \textbf{target-aware posterior} distribution,
$q_\phi(\mathbf{Z} \mid \mathbf{x}, \mathbf{y})$.

By conditioning the posterior on both the input $\mathbf{x}$ and the ground-truth answer $\mathbf{y}$, we provide the model with a powerful, self-contained training signal. The model can effectively use ``hindsight'' to infer the latent reasoning states $\mathbf{Z}$ that would have been most useful for bridging the gap between the question and the final answer, allowing it to discover optimal reasoning pathways organically.

\textbf{The Inherent Train-Inference Mismatch.}
While powerful, this approach leads to the standard Evidence Lower Bound (ELBO):
\begin{equation}
\log p_\theta(\mathbf{y} \mid \mathbf{x})
\;\geq\;
\mathbb{E}_{q_\phi(\mathbf{Z} \mid \mathbf{x},\mathbf{y})}
\!\left[
\log p_\theta(\mathbf{y} \mid \mathbf{x}, \mathbf{Z})
\right]
-
D_{\mathrm{KL}}
\!\left(
q_\phi(\mathbf{Z} \mid \mathbf{x},\mathbf{y})
\,\|\, 
p_\theta(\mathbf{Z} \mid \mathbf{x})
\right).
\end{equation}
The first term trains the decoder to generate $\mathbf{y}$ conditioned on latent reasoning states, while the KL term encourages the prior to match the posterior. However, in our setting the posterior has access to the target sequence $\mathbf{y}$ during training, whereas the prior at inference depends only on $\mathbf{x}$. This asymmetry creates a critical failure mode we term \textbf{answer leakage}: rather than grounding the latent reasoning trace in the multimodal input, the posterior $q_\phi(\mathbf{Z} \mid \mathbf{x},\mathbf{y})$ can take a shortcut by relying almost entirely on the reference answer $\mathbf{y}$ to construct $\mathbf{Z}$~\cite{bowman2016generating, zhao2017learning}. Because the standard ELBO uses the forward KL term to pull the prior toward this posterior, the prior is trained to mimic these answer-dependent latent states. At inference time, however, the reference answer is unavailable, and the prior is left navigating a latent space whose structure was shaped by shortcuts it can no longer exploit. The result is a severe \textbf{train-inference mismatch}: the prior fails to produce latent states that support useful reasoning, despite having been trained to approximate the posterior. This mismatch is especially acute in strong autoregressive decoders, where the high model capacity makes answer leakage easy---the posterior can encode target information through subtle statistical dependencies that are nearly invisible to the reconstruction loss, yet inaccessible to the prior at test time.

\textbf{Theoretical Analysis of Answer Leakage.}
To formalize this failure mode, we decompose the posterior mean into an input-grounded component and an answer-dependent shift:
$\mu_\phi(\mathbf{x},\mathbf{y}) = \mu^{(x)}_\phi(\mathbf{x}) + \delta(\mathbf{x},\mathbf{y})$,
where $\delta(\mathbf{x},\mathbf{y})$ denotes the posterior displacement induced by the target answer $\mathbf{y}$. Under this decomposition, one can show (Proposition~\ref{prop:prior_contamination} in Appendix~\ref{appendix:prior_contamination_proof}) that minimizing the standard ELBO KL term causes the learned prior mean to absorb the average answer-dependent shift
$\bar{\delta}(\mathbf{x}) := \mathbb{E}_{\mathbf{y} \sim p(\mathbf{y}\mid \mathbf{x})}[\delta(\mathbf{x},\mathbf{y})]$
as a residual bias, producing a prior contamination of magnitude
$1/2 \sum_{j=1}^d \big[ \bar{\delta}_j(\mathbf{x})^2 / \sigma^2_{\theta,j}(\mathbf{x}) \big]$.
Furthermore, one can show (Proposition~\ref{prop:fwd_insufficient} in Appendix~\ref{appendix:fwd_same_optimum_proof}) that forward KL alignment alone---even with a stop-gradient on the posterior---yields the same contaminated prior optimum and exerts no direct corrective gradient on the posterior itself: the posterior's answer-dependent bias is unchanged. See Appendix~\ref{appendix:theory_amvl_leakage} for full statements and proofs.

\subsection{Asymmetric Mutual Variational Learning (AMVL)}

To resolve this fundamental mismatch, we introduce \textbf{Asymmetric Mutual Variational Learning (AMVL)}, a framework that establishes a mutual learning process between the prior and posterior through a bidirectional KL regularization mechanism. While standard ELBO training uses a one-sided objective that can cause the prior to inherit the posterior's answer-leakage bias, as theoretically analyzed in Appendix~\ref{appendix:theory_amvl_leakage} (Propositions~\ref{prop:prior_contamination} and~\ref{prop:fwd_insufficient}), AMVL calibrates the latent space from both directions. This dual-calibration approach is designed to make the latent space both expressive during training and well-calibrated for inference.

AMVL achieves this with two complementary losses that compose its bidirectional calibration mechanism:
\begin{itemize}[leftmargin=*, itemsep=0.2pt, topsep=0pt, partopsep=0pt]
    \item \textbf{Forward Prior Alignment (\(\mathcal{L}_{\text{fwd}}\)).} To improve the prior's calibration, we use a forward KL term (derived from the standard Evidence Lower Bound; see Appendix~B.1) to train the prior \(p_\theta\) to match the latent states inferred by the posterior \(q_\phi\). We use a stop-gradient (\(\mathrm{sg}[\cdot]\)) on the posterior to ensure this loss updates only the prior parameters:
    \begin{equation}
    \mathcal{L}_{\text{fwd}} =
    D_{\text{KL}}\big(
    \mathrm{sg}[q_\phi(\mathbf{Z}\mid \mathbf{x}, \mathbf{y})]
    \,\|\,
    p_\theta(\mathbf{Z}\mid \mathbf{x})
    \big).
    \end{equation}
    This term effectively teaches the inference-time prior to approximate the latent reasoning states found to be useful during training.

    \item \textbf{Reverse Posterior Regularization (\(\mathcal{L}_{\text{rev}}\)).} Forward alignment alone is insufficient, as it does not prevent the posterior from exploiting answer leakage in the first place. We therefore introduce a reverse KL term that updates only the posterior, regularizing it against drifting into prior-incompatible regions:
    \begin{equation}
    \mathcal{L}_{\text{rev}} =
    D_{\text{KL}}\big(
    \mathrm{sg}[p_\theta(\mathbf{Z}\mid \mathbf{x})]
    \,\|\,
    q_\phi(\mathbf{Z}\mid \mathbf{x}, \mathbf{y})
    \big).
    \end{equation}
    This term directly penalizes posterior drift away from high-density regions of the prior. Proposition~\ref{prop:reverse_restoring} formalizes this effect for the common diagonal Gaussian case.
\end{itemize}

\begin{proposition}[Reverse KL suppresses incompatible posterior drift]
\label{prop:reverse_restoring_main}
For diagonal Gaussian latents
$p_\theta(\mathbf{z} \mid \mathbf{x}) = \mathcal{N}\big(\mu_\theta(\mathbf{x}), \mathrm{diag}(\sigma^2_\theta(\mathbf{x}))\big)$
and
$q_\phi(\mathbf{z} \mid \mathbf{x}, \mathbf{y}) = \mathcal{N}\big(\mu_\phi(\mathbf{x}, \mathbf{y}), \mathrm{diag}(\sigma^2_\phi(\mathbf{x}, \mathbf{y}))\big)$,
the reverse KL in Eq.~(4) has the closed form
\[
\mathcal{L}_{\text{rev}} =
\frac{1}{2} \sum_j \left(
\frac{\sigma^2_{\theta,j}}{\sigma^2_{\phi,j}}
+
\frac{(\mu_{\theta,j} - \mu_{\phi,j})^2}{\sigma^2_{\phi,j}}
-
1
+
\log\frac{\sigma^2_{\phi,j}}{\sigma^2_{\theta,j}}
\right),
\]
where \(p_\theta\) is treated as constant by stop-gradient. Its gradient with respect to the posterior mean is
\[
\frac{\partial \mathcal{L}_{\text{rev}}}{\partial \mu_{\phi,j}}
=
\frac{\mu_{\phi,j} - \mu_{\theta,j}}{\sigma^2_{\phi,j}}.
\]
Thus, reverse KL penalizes posterior drift from prior-compatible high-density latent regions, with a mean-mismatch penalty stronger as the posterior sharpens (i.e., as $\sigma^2_{\phi,j}$ decreases).
\end{proposition}
\begin{proof}
This result is proved in Appendix~\ref{appendix:reverse_restoring_proof}.
\end{proof}
Combining $\mathcal{L}_{\text{fwd}}$ and $\mathcal{L}_{\text{rev}}$ creates bidirectional calibration. While $\mathcal{L}_{\text{fwd}}$ pulls the prior toward the posterior, $\mathcal{L}_{\text{rev}}$ prevents the posterior from drifting into inference-incompatible states. Proposition~\ref{prop:amvl_shrinkage} formalizes how this objective reduces prior contamination caused by answer leakage.

\begin{proposition}[Bidirectional calibration reduces prior contamination]
\label{prop:amvl_shrinkage_main}
Under the local linear leakage model in Assumption~\ref{assump:linear_leakage}, where the answer-dependent shift is written as
\[
\delta(\mathbf{x}, \mathbf{y}) = \alpha f(\mathbf{x}, \mathbf{y}),
\]
and under the local linear-response stationary condition in Assumption~\ref{assump:linear_response}, there exists a constant \(c > 0\) such that the local equilibrium leakage coefficient \(\alpha\) under AMVL satisfies
\[
\alpha_{\text{AMVL}}
=
\frac{\alpha_{\text{ELBO}}}{1 + c\gamma/\sigma^2_{\text{eff}}},
\]
where \(\alpha_{\text{ELBO}}\) is the equilibrium leakage under one-sided ELBO training, \(\gamma\) is the weight of the reverse KL term, and \(\sigma^2_{\text{eff}}\) is an effective posterior variance. Consequently, the prior mean contamination under AMVL satisfies
\[
\Delta_{\text{AMVL}}(\mathbf{x})
=
\|\mu_\theta^{\text{AMVL}}(\mathbf{x}) - \mu^{(x)}(\mathbf{x})\|_2
=
\alpha_{\text{AMVL}}\|\bar{f}(\mathbf{x})\|_2,
\]
whereas the corresponding contamination under one-sided ELBO matching is
\[
\Delta_{\text{ELBO}}(\mathbf{x})
=
\alpha_{\text{ELBO}}\|\bar{f}(\mathbf{x})\|_2.
\]
Therefore,
\[
\Delta_{\text{AMVL}}(\mathbf{x}) < \Delta_{\text{ELBO}}(\mathbf{x})
\]
whenever \(\gamma > 0\), \(\|\bar{f}(\mathbf{x})\|_2 > 0\), and \(\alpha_{\text{ELBO}} > 0\).
\end{proposition}
\begin{proof}
This result is proved in Appendix~\ref{appendix:amvl_shrinkage_proof}.
\end{proof}
Proposition~\ref{prop:amvl_shrinkage} formalizes how this dual objective reduces the prior contamination caused by answer leakage. The forward KL term then transfers this less-contaminated signal to the prior, resulting in a model that is better calibrated for inference. From a geometric perspective, these two KL directions establish a mutual mass-covering dynamic. As detailed in Appendix~B.3, both KL terms are mass-covering because optimization occurs on the second argument of the divergence. The forward term forces the prior to cover the posterior, while the reverse term forces the posterior to cover the prior, discouraging it from collapsing into narrow, inference-incompatible regions. This dual calibration is the core mechanism allowing AMVL to discover an expressive yet inference-compatible latent reasoning space. Further theoretical discussion, including an information-theoretic view in Proposition~\ref{prop:forward_kl_decomposition}, is available in Appendix~\ref{appendix:theory_amvl_leakage}.

\subsection{Implementation within an MLLM}

We instantiate AMVL by augmenting a standard MLLM architecture with minimal, lightweight components.

\textbf{Latent-Integrated Architecture.}
To integrate latent variables into the autoregressive process, we insert $k$ placeholder tokens, denoted as \texttt{<latent>}, into the input sequence between the prompt $\mathbf{x}$ and target $\mathbf{y}$:
\begin{equation}
S = [\mathbf{x}, \texttt{<latent>}_1, \dots, \texttt{<latent>}_k, \mathbf{y}].
\label{eq:latent_sequence}
\end{equation}
These placeholders serve a dual role: their final hidden states provide the contextual information needed to parameterize the latent distributions, and they are then replaced by sampled latent features to condition the decoder for answer generation. We model both the prior and posterior as factorized diagonal Gaussians for computational efficiency and stable, closed-form optimization:
\begin{equation}
p_\theta(\mathbf{Z}|\mathbf{x}) = \prod_{i=1}^k \mathcal{N}\left(\mathbf{z}_i; \boldsymbol{\mu}_{\theta,i}(\mathbf{x}), \mathrm{diag}(\boldsymbol{\sigma}^2_{\theta,i}(\mathbf{x}))\right),
\label{eq:prior_gaussian}
\end{equation}
\begin{equation}
q_\phi(\mathbf{Z}|\mathbf{x}, \mathbf{y}) = \prod_{i=1}^k \mathcal{N}\left(\mathbf{z}_i; \boldsymbol{\mu}_{\phi,i}(\mathbf{x}, \mathbf{y}), \mathrm{diag}(\boldsymbol{\sigma}^2_{\phi,i}(\mathbf{x}, \mathbf{y}))\right).
\label{eq:posterior_gaussian}
\end{equation}

\textbf{LLM-Native Variational Head.}
To seamlessly integrate this machinery into the MLLM, we design lightweight variational heads that operate on the model's hidden states. Let $H = [h_1, \ldots, h_k] \in \mathbb{R}^{k \times D}$ be the final-layer hidden states at the $k$ latent token positions. The Gaussian parameters are computed via a projection network:
\begin{equation}
[\boldsymbol{\mu}, \log \boldsymbol{\sigma}^2] = W_{\text{out}} \, \mathrm{SwiGLU}(\mathrm{RMSNorm}(\mathbf{H})),
\label{eq:variational_head}
\end{equation}
where $\mu, \log \sigma^2 \in \mathbb{R}^{k \times d}$. This ``LLM-native'' design leverages standard components~\cite{shazeer2020glu, zhang2019root}, preserving architectural consistency. The same head is used for both the prior and posterior, differing only in whether the MLLM's input context includes the target $\mathbf{y}$. Unless otherwise stated, we set the latent dimension to $d=512$.

\textbf{Latent-Conditioned Decoding.}
We use the standard reparameterization trick~\cite{rezende2014stochastic} to sample a latent sequence $\mathbf{Z}$:
\begin{equation}
\mathbf{Z} = \boldsymbol{\mu} + \boldsymbol{\epsilon} \odot \boldsymbol{\sigma}, \quad \boldsymbol{\epsilon} \sim \mathcal{N}(0, \mathbf{I}).
\label{eq:reparameterization}
\end{equation}
The sampled $\mathbf{Z} \in \mathbb{R}^{k \times d}$ is then projected back to the MLLM's hidden dimension $D$ via a \textbf{latent injector} and used to replace the embeddings at the \texttt{<latent>} positions, conditioning the final answer generation on the continuous reasoning states.

\subsection{Final Training Objective}

We optimize the entire model end-to-end by minimizing a composite loss function:
\begin{equation}
\mathcal{L}_\text{total} = \mathcal{L}_\text{NTP} + \beta \mathcal{L}_\text{fwd} + \gamma \mathcal{L}_\text{rev},
\label{eq:total_loss}
\end{equation}
where $\beta$ and $\gamma$ are scalar weights, and $\mathcal{L}_{\text{NTP}}$ is the standard autoregressive next-token prediction loss, conditioned on a latent sample from the posterior:
\begin{equation}
\mathcal{L}_\text{NTP} = -\mathbb{E}_{\mathbf{Z} \sim q_\phi(\mathbf{Z}|\mathbf{x},\mathbf{y})} \left[\sum_t \log p_\text{LLM}(y_t | \mathbf{x}, y_{<t}, \mathbf{Z})\right].
\label{eq:ntp_loss}
\end{equation}

\textbf{Closed-Form KL Computation.}
A key advantage of our diagonal Gaussian parameterization is that both KL terms admit a closed-form solution. For two diagonal Gaussians $\mathcal{N}_1 = \mathcal{N}(\mu_1, \mathrm{diag}(\sigma_1^2))$ and $\mathcal{N}_2 = \mathcal{N}(\mu_2, \mathrm{diag}(\sigma_2^2))$, the KL divergence is:
\begin{equation}
D_{KL}(\mathcal{N}_1 \| \mathcal{N}_2) = \frac{1}{2} \sum_{j=1}^d \left( \frac{\sigma^2_{1,j} + (\mu_{1,j} - \mu_{2,j})^2}{\sigma^2_{2,j}} - 1 + \log \frac{\sigma^2_{2,j}}{\sigma^2_{1,j}} \right).
\label{eq:gaussian_kl}
\end{equation}
This allows for efficient and stable computation of both $\mathcal{L}_\text{fwd}$ and $\mathcal{L}_\text{rev}$ during training.

\textbf{Asymmetric KL Scheduling.}
Rather than using fixed weights, we schedule $\beta$ and $\gamma$ to stabilize training, as detailed in Appendix~\ref{appendix:experimental_details}. The forward KL weight $\beta$ is warmed up early to allow the prior to begin tracking the posterior. The reverse KL weight $\gamma$ is introduced with a delay and kept weaker. This strategy allows the prior to become partially calibrated before the reverse regularization constrains the posterior, preventing over-regularization in the early stages when the prior is still weak.

\textbf{Inference.}
At inference time, the target $\mathbf{y}$ is unavailable. We therefore sample latent variables from the learned prior $\mathbf{Z} \sim p_\theta(\mathbf{Z}|\mathbf{x})$, inject them into the decoder, and generate the answer autoregressively. The bi-bounded training ensures this prior is well-calibrated to produce latent states that effectively guide the model toward the correct answer.


\section{Experiments}
\subsection{Experimental Setup}
\textbf{Implementation details.}
We implement AMVL on top of Qwen2.5-VL-7B-Instruct~\cite{qwen25}. Training uses a mixture of multimodal reasoning datasets, including Visual-CoT~\cite{Visualcot}, ReFocus~\cite{ReFocus}, CogCoM~\cite{CogCoM}, and Zebra-CoT~\cite{Zebra-CoT}. Unless otherwise specified, we use $k=8$ latent tokens and latent dimension $d=512$. Following~\cite{lvr, monet}, the vision encoder is frozen during training, while the language backbone and variational modules are jointly optimized. 
Full details on preprocessing, latent token construction, KL scheduling, optimization, and infrastructure are provided in Appendix~\ref{appendix:experimental_details}.

\textbf{Experiment Benchmarks.}
To rigorously evaluate our method, we structure our experiments across three distinct axes: fine-grained visual perception, complex visual reasoning, and out-of-distribution (OOD) robustness. First, to assess \textbf{high-resolution processing and dense feature extraction}, we evaluate on V$^*$~\cite{vstar}, HRBench4K~\cite{hrbench}, and HRBench8K~\cite{hrbench}. Second, we investigate \textbf{core visual cognition and spatial logic} using the diverse perception-heavy tasks of the BLINK~\cite{blink} benchmark. Finally, to examine robustness beyond in-distribution settings, we assess OOD generalization on the abstract reasoning categories of the VisualPuzzles benchmark (results in Appendix~\ref{appendix:ood_visualpuzzles}).

\textbf{Baselines.} To ensure a fair comparison, we select representative state-of-the-art MLLM baselines built upon the same foundational model, Qwen2.5-VL-7B. We group these baselines into two paradigms: (1) \textbf{Discrete Reasoning}, which includes text-centric approaches (\emph{thinking about images}, e.g., Vision-R1~\cite{visionr1}, PAPO~\cite{papo}) and visually-augmented discrete generation (\emph{thinking with images}, e.g., PixelReasoner~\cite{pixelreasoner}, DeepEyes~\cite{deepeyes}); and (2) \textbf{Continuous Latent Reasoning}, which includes LVR~\cite{lvr}, Mull-Tokens~\cite{mull}, and Monet~\cite{monet}. Detailed descriptions are provided in Appendix~\ref{app:baseline}.

\begin{table}[t]
\caption{Performance on V$^*$, HRBench4K, and HRBench8K. ``Avg.'' is the mean Overall score across benchmarks. Best and second-best open-source results are highlighted in \textbf{bold} and \underline{underlined}. The last row shows absolute gains over Qwen2.5-VL-7B.}
\label{tab:main_results}
\centering
\resizebox{0.99\textwidth}{!}{
\begin{tabular}{lcccccccccc}
\toprule
\multirow{2}{*}{\textbf{Method}} &  \multicolumn{3}{c}{\textbf{V$^*$}} & \multicolumn{3}{c}{\textbf{HRBench4K}} & \multicolumn{3}{c}{\textbf{HRBench8K}} & \multirow{2}{*}{\textbf{Avg.}}\\
\cmidrule(lr){2-4} \cmidrule(lr){5-7} \cmidrule(lr){8-10}
& Overall & Attribute & Spatial & Overall & FSP & FCP & Overall & FSP & FCP \\
\midrule
\multicolumn{11}{c}{\textit{Proprietary Models}} \\
\midrule
GPT-4o        & 67.50 & 72.20 & 60.50 & 59.00 & 70.00 & 48.00 & 55.50 & 62.00 & 49.00 & 60.67 \\
\midrule
\multicolumn{11}{c}{\textit{Open-Source Models}} \\
\midrule
Qwen2.5-VL-7B & 76.44 & 77.39 & 75.00 & 68.00 & 80.25 & 55.75 & 63.75 & 74.25 & 53.25 & 69.40 \\
\quad+ SFT    & 81.68 & \underline{83.48} & 78.95 & 68.38 & 78.28 & 58.50 & 61.63 & 70.75 & 52.50 & 70.56 \\
\quad+ SFT + GRPO & 78.53 & 78.26 & 78.95 & 70.00 & 83.25 & 56.75 & 66.75 & 78.00 & 55.50 & 71.76 \\
DeepEyes      & \underline{83.25} & \textbf{84.35} & 81.58 & 71.25 & 83.75 & \underline{58.75} & 65.13 & 77.00 & 53.25 & 73.21 \\
LVR           & 81.15 & 80.00 & \underline{82.89} & 70.62 & 83.50 & 57.75 & 64.12 & 77.25 & 51.00 & 71.96 \\
PixelReasoner & 81.15 & 80.87 & 81.58 & \underline{72.00} & 84.00 & \textbf{60.00} & 66.12 & 76.50 & \underline{55.75} & 73.09 \\
Mull-Tokens          & 79.06 & 81.58 & 77.49 & 70.25 & \underline{86.50} & 54.00 & 65.75 & \underline{81.00} & 50.50 & 71.69 \\
Monet         & \underline{83.25} & \underline{83.48} & \underline{82.89} & 71.00 & 85.25 & 56.75 & \underline{68.00} & 79.75 & \textbf{56.25} & \underline{74.08} \\
\midrule
Ours-7B       & \textbf{84.29} & \textbf{84.35} & \textbf{84.21} & \textbf{72.12} & \textbf{87.25} & 57.00 & \textbf{68.50} & \textbf{83.50} & 53.50 & \textbf{74.97} \\
Absolute gain & {\color{green!60!black}+7.85} & {\color{green!60!black}+6.96} & {\color{green!60!black}+9.21} & {\color{green!60!black}+4.12} & {\color{green!60!black}+7.00} & {\color{green!60!black}+1.25} & {\color{green!60!black}+4.75} & {\color{green!60!black}+9.25} & {\color{green!60!black}+0.25} & {\color{green!60!black}+5.57} \\
\bottomrule
\end{tabular}
}
\vspace{-7pt}
\end{table}

\begin{table}[t]
\caption{Performance on vision-centric BLINK tasks. 
Obj. Loc., M-View, and Vis. Sim. denote Object Localization, Multi-view Reasoning, and Visual Similarity.
}
\label{tab:blink_results}
\centering
\resizebox{0.99\textwidth}{!}{
\begin{tabular}{lcccccccc}
\toprule
Method & IQ Test & Jigsaw & Obj. Loc. & Art Style & M-View & Spatial Relation & Vis. Sim. & Avg. \\
\midrule
Qwen2.5-VL-7B & 20.00 & 45.33 & 49.18 & 64.96 & 41.35 & \underline{88.81} & 82.96 & 56.08 \\
Vision-R1     & 27.33 & 54.67 & 48.36 & 65.81 & 45.86 & 79.72 & 71.85 & 56.23 \\
PixelReasoner & 26.00 & \underline{72.00} & \underline{54.10} & 67.52 & 46.62 & 86.71 & \underline{85.93} & \underline{62.70} \\
PAPO          & 22.67 & 66.67 & \textbf{56.56} & 64.96 & 49.62 & \textbf{90.21} & 85.19 & 62.27 \\
LVR           & 26.00 & 52.67 & 50.00 & \textbf{70.94} & 45.41 & \underline{88.81} & 82.22 & 59.44 \\
Mull-Tokens    & \underline{31.33} & 69.33 & 47.54 & 62.42 & \textbf{60.15} & 84.61 & 80.74 & 62.30 \\
Monet         & 29.33 & 42.67 & 52.46 & 62.39 & 50.38 & 79.02 & 85.19 & 57.35 \\
\midrule
Ours-7B       & \textbf{32.67} & \textbf{77.33} & \underline{54.92} & \underline{70.09} & \underline{55.64} & \underline{88.81} & \textbf{88.89} & \textbf{66.91} \\
Absolute Gain & {\color{green!60!black}+12.67} & {\color{green!60!black}+32.00} & {\color{green!60!black}+5.74} & {\color{green!60!black}+5.13} & {\color{green!60!black}+14.29} & {\color{gray}+0.00} & {\color{green!60!black}+5.93} & {\color{green!60!black}+10.83} \\
\bottomrule
\end{tabular}
}
\vspace{-7pt}
\end{table}

\subsection{Main Results}

\textbf{Overall Performance Analysis.}
AMVL establishes a new state-of-the-art among Qwen2.5-VL-7B based models. On fine-grained perception benchmarks (Table~\ref{tab:main_results}), it achieves a 74.97 average (+5.57 absolute gain), driven by robust dense feature extraction on V$^*$ (+7.85) and HRBench8K (+4.75). Furthermore, AMVL demonstrates exceptional proficiency in complex spatial reasoning (Table~\ref{tab:blink_results}), elevating the BLINK average by +10.83. This is highlighted by a remarkable +32.00 surge on the topology-heavy Jigsaw task. Ultimately, these gains confirm that bidirectionally regularized latent reasoning fundamentally enhances both microscopic visual search and macroscopic structural logic. This advantage further extends to out-of-distribution settings, where AMVL consistently outperforms competing methods on shifted visual puzzle variants (Appendix~\ref{appendix:ood_visualpuzzles}), indicating that the learned latent reasoning space remains calibrated and reliable under distribution shift.

\textbf{Compared with Discrete Multimodal Reasoning.} 
To validate our hypothesis regarding the \emph{language-space bottleneck}, we compare AMVL against recent state-of-the-art discrete reasoning models, including methods that rely on text-based reasoning chains (\emph{thinking about images}, e.g., Vision-R1, PAPO) and those that interleave visual tools within discrete generation (\emph{thinking with images}, e.g., PixelReasoner, DeepEyes). While these baselines improve upon the base model, they force high-dimensional visual concepts into a discrete, lossy text space.
As shown in Tables~\ref{tab:main_results} and \ref{tab:blink_results}, AMVL consistently outperforms these discrete paradigms. On abstract tasks like IQ Test and Visual Similarity in BLINK, where discrete verbalization often falls short or induces hallucinations, AMVL achieves superior accuracy (32.67 and 88.89, respectively). This empirically validates our claim in the Introduction: performing intermediate reasoning directly in continuous perceptual space preserves essential spatial and abstract representations that discrete language tokens cannot fully capture.

\textbf{Compared with Continuous Latent Reasoning.} 
More importantly, we benchmark AMVL against pioneering \emph{latent reasoning} frameworks, including LVR, Mull-Tokens, and Monet. While these methods bypass the language bottleneck, they rely on \emph{explicit, hand-crafted supervision signals} (such as predefined reconstruction targets), which restricts the model's ability to discover optimal intermediate representations. 
AMVL's systematic superiority over LVR, Mull-Tokens, and Monet across virtually all sub-metrics provides strong empirical evidence for our variational formulation. By framing continuous reasoning as a structured probabilistic inference problem, AMVL eliminates the need for latent supervision. The superior performance—particularly the consistent gains across in-distribution benchmarks—demonstrates that our dual-KL regularization mitigates posterior collapse and improves the quality of the learned latent reasoning space. It ensures that the continuous reasoning representations are not only expressive during training but also stable and effective during inference.

\subsection{Ablation Study}


\begin{table}[t]
\caption{Ablation of training objectives. ``Fwd-KL'' and ``Rev-KL'' denote the forward and reverse KL terms. The last two rows differ only in optimization order.}
\label{tab:loss_ablation}
\centering
\begin{tabular}{lccc}
\toprule
Method & V$^*$ & HRBench4K & HRBench8K \\
\midrule
NTP only & 81.15 & 70.50 & 67.38 \\
Fwd-KL only & 40.84 & 53.37 & 52.00 \\
Rev-KL only & 75.92 & 69.62 & 64.38 \\
NTP + Fwd-KL & 82.72 & 72.12 & 67.75 \\
NTP + Rev-KL & 82.20 & 71.75 & 67.25 \\
NTP + Fwd-KL + Rev-KL (reverse-first) & 80.63 & 72.38 & 68.25 \\
NTP + Fwd-KL + Rev-KL (forward-first) & 84.29 & 72.12 & 68.50 \\
\bottomrule
\end{tabular}
\vspace{-7pt}
\end{table}



\textbf{Effect of bidirectional variational calibration.}
Table~\ref{tab:loss_ablation} shows the next-token prediction (NTP) baseline suffers from a train-inference mismatch. \textbf{As measured in our latent-spread analysis (Appendix~\ref{appendix:latent_spread})}, either calibration direction partially alleviates this issue: forward alignment improves the inference-time prior by training it to track posterior latents, while reverse support regularization constrains posterior drift toward regions unsupported by the prior. However, optimizing either direction alone is suboptimal. Forward-only training still leaves the posterior unconstrained and target-dependent, whereas reverse-only training can \textbf{overly suppress} useful latent variability despite achieving tight geometric compatibility. The full AMVL objective performs best, confirming that effective multimodal latent reasoning requires both prior alignment and posterior regularization.

\textbf{Effect of latent configuration.}
Table~\ref{tab:latent_config_ablation} studies varying latent token \textbf{count} $k$ and dimension $d$. Increasing $k$ from 4 to 8 improves performance, providing the minimum capacity necessary for complex reasoning. However, further increasing $k$ to 16 degrades results. Unlike discrete text tokens---which suffer from a low information-per-token ratio---continuous vectors possess \textbf{inherently high information density}, encapsulating rich logic compactly. Consequently, excessive latent slots introduce redundancy and ``information dilution,'' complicating variational optimization. A similar trend holds for $d$: increasing $d$ from 512 to 768 degrades \textbf{performance}, as over-parameterization exacerbates dual-KL optimization burden without adding reasoning power. Given the efficiency of high-density latent representations, we adopt $k=8$ and $d=512$ as the default AMVL configuration.

\textbf{Additional ablations.}
Appendix~\ref{app:more_ablation} provides additional ablations on variational head architectures, loss weights, stop-gradient design, and inference-time latent sampling. Results show that AMVL benefits from a lightweight native variational head, remains stable across loss weights, relies on decoupled gradients for effective prior-posterior alignment, and is robust to inference-time perturbations.

\begin{table}[t]
\caption{Ablation study of latent configuration. We vary the number of latent tokens $k$ and the latent dimension $d$ while keeping other settings fixed.}
\label{tab:latent_config_ablation}
\centering
\begin{tabular}{lccccccc}
\toprule
\multirow{2}{*}{Setting} & \multicolumn{3}{c}{{Number of latent tokens $k$}} & \multicolumn{4}{c}{{Latent dimension $d$}} \\
\cmidrule(lr){2-4} \cmidrule(lr){5-8}
 & 4 tokens & 8 tokens & 16 tokens & 128 & 256 & 512 & 768 \\
\midrule
HRBench4K & 70.00 & 72.12 & 71.75 & 72.62 & 71.50 & 72.12 & 70.88 \\
HRBench8K & 66.25 & 68.50 & 67.12 & 66.75 & 66.62 & 68.50 & 67.38 \\
V$^*$ & 76.44 & 84.29 & 81.15 & 82.72 & 76.96 & 84.29 & 79.58 \\
\bottomrule
\end{tabular}
\vspace{-7pt}
\end{table}

\section{Conclusion}

In this paper, we address the language-space bottleneck and train–inference mismatch in multimodal large language models. We propose \textbf{Asymmetric Mutual Variational Learning (AMVL)}, a principled framework for continuous latent reasoning that jointly aligns the prior and regularizes the posterior. AMVL decouples latent expressiveness from answer leakage without requiring hand-crafted supervision. Experiments on fine-grained perception and abstract reasoning benchmarks show that AMVL achieves state-of-the-art performance over both discrete Chain-of-Thought and prior latent-reasoning methods.


\bibliographystyle{unsrt}
\bibliography{reference}

\newpage
\appendix
\counterwithin{proposition}{section}

\section{Main Notation Introduction}

For clarity, Table~\ref{tab:main_notation} summarizes the main notations used throughout the paper.

\begin{table*}[htbp]
\centering
\caption{Meanings of the main notations used in the paper.}
\label{tab:main_notation}
\renewcommand{\arraystretch}{1.15}
\setlength{\tabcolsep}{8pt}
\resizebox{0.99\textwidth}{!}{
\begin{tabular}{l l p{9.5cm}}
\toprule
\textbf{Notation} & \textbf{Type} & \textbf{Meaning} \\
\midrule
$\mathbf{x}$ & sequence / context & Multimodal input context, e.g., visual tokens and text prompt. \\
$\mathbf{y}$ & sequence & Target output sequence (answer tokens). \\
$y_t$ & token & The $t$-th token in the target sequence. \\
$\mathbf{y}_{<t}$ & sequence prefix & Target prefix before time step $t$. \\
$\mathbf{Z}=[\mathbf{z}_1,\dots,\mathbf{z}_k]$ & $\mathbb{R}^{k\times d}$ & Continuous latent reasoning sequence consisting of $k$ latent slots. \\
$\mathbf{z}_i$ & $\mathbb{R}^{d}$ & The continuous latent variable at the $i$-th latent slot. \\
$k$ & integer & Number of latent slots. \\
$d$ & integer & Latent dimension of each latent slot. \\
$D$ & integer & Hidden dimension of the base MLLM. \\
$\mathbf{H}=[\mathbf{h}_1,\dots,\mathbf{h}_k]$ & $\mathbb{R}^{k\times D}$ & Final-layer hidden states at the latent placeholder positions. \\
$\mathbf{h}_i$ & $\mathbb{R}^{D}$ & Hidden state at the $i$-th latent token position. \\
$p_\theta(\mathbf{Z}|\mathbf{x})$ & distribution & Target-agnostic prior distribution over latent reasoning states, conditioned solely on input $\mathbf{x}$. \\
$q_\phi(\mathbf{Z}|\mathbf{x},\mathbf{y})$ & distribution & Target-aware variational posterior distribution over latent reasoning states, conditioned on $(\mathbf{x},\mathbf{y})$. \\
$\theta$ & parameters & Parameters of the prior branch and the autoregressive decoder. \\
$\phi$ & parameters & Parameters of the variational posterior branch. \\
$\boldsymbol{\mu}_{\theta,i}(\mathbf{x})$ & $\mathbb{R}^{d}$ & Mean vector of the prior distribution for the $i$-th latent slot. \\
$\boldsymbol{\sigma}_{\theta,i}^2(\mathbf{x})$ & $\mathbb{R}^{d}$ & Diagonal variance vector of the prior distribution for the $i$-th latent slot. \\
$\boldsymbol{\mu}_{\phi,i}(\mathbf{x},\mathbf{y})$ & $\mathbb{R}^{d}$ & Mean vector of the posterior distribution for the $i$-th latent slot. \\
$\boldsymbol{\sigma}_{\phi,i}^2(\mathbf{x},\mathbf{y})$ & $\mathbb{R}^{d}$ & Diagonal variance vector of the posterior distribution for the $i$-th latent slot. \\
$\boldsymbol{\mu}$ & $\mathbb{R}^{k\times d}$ & Mean tensor of a given latent Gaussian distribution. \\
$\boldsymbol{\sigma}^2$ & $\mathbb{R}^{k\times d}$ & Diagonal variance tensor of a given latent Gaussian distribution. \\
$\log \boldsymbol{\sigma}^2$ & $\mathbb{R}^{k\times d}$ & Log-variance tensor predicted by the variational heads. \\
$\boldsymbol{\epsilon}$ & $\mathbb{R}^{k\times d}$ & Standard Gaussian noise used in the reparameterization trick, sampled from $\mathcal{N}(0,\mathbf{I})$. \\
$\mathbf{H}_Z$ & $\mathbb{R}^{k\times D}$ & Decoder-aligned latent features obtained by projecting $\mathbf{Z}$ through the latent injector. \\
$\mathrm{Injector}(\cdot)$ & mapping & Linear or MLP projection from latent space $\mathbb{R}^{d}$ to MLLM hidden space $\mathbb{R}^{D}$. \\
$\mathrm{sg}[\cdot]$ & operator & Stop-gradient operator, used to block gradient flow through a specific branch. \\
$\mathcal{L}_{\mathrm{fwd}}$ & scalar loss & Forward KL alignment loss, $D_{\mathrm{KL}}(\mathrm{sg}[q_\phi(\mathbf{Z}|\mathbf{x},\mathbf{y})]\|p_\theta(\mathbf{Z}|\mathbf{x}))$. \\
$\mathcal{L}_{\mathrm{rev}}$ & scalar loss & Reverse KL regularization loss, $D_{\mathrm{KL}}(\mathrm{sg}[p_\theta(\mathbf{Z}|\mathbf{x})]\|q_\phi(\mathbf{Z}|\mathbf{x},\mathbf{y}))$. \\
$\mathcal{L}_{\mathrm{NTP}}$ & scalar loss & Autoregressive next-token prediction loss conditioned on latent reasoning states. \\
$\mathcal{L}_{\mathrm{total}}$ & scalar loss & Final training objective combining $\mathcal{L}_{\mathrm{NTP}}$, $\mathcal{L}_{\mathrm{fwd}}$, and $\mathcal{L}_{\mathrm{rev}}$. \\
$\beta$ & scalar & Scheduling weight for the forward KL loss term. \\
$\gamma$ & scalar & Scheduling weight for the reverse KL loss term. \\
$D_{\mathrm{KL}}(\cdot\|\cdot)$ & divergence & Kullback--Leibler divergence between two distributions. \\
\bottomrule
\end{tabular}
}
\end{table*}

\section{Variational Derivations and Theoretical Motivation}
\label{appendix:variational_derivation}

\subsection{Derivation of the ELBO}
\label{appendix:elbo_derivation}

We begin with the conditional log-marginal likelihood of the target sequence $\mathbf{y}$ given the multimodal context $\mathbf{x}$:
\begin{equation}
\log p_\theta(\mathbf{y}|\mathbf{x})
=
\log \int p_\theta(\mathbf{y},\mathbf{Z}|\mathbf{x})\, d\mathbf{Z},
\end{equation}
where $\mathbf{Z}$ denotes the latent reasoning sequence. This quantity is the ideal training objective, since it marginalizes over all possible latent reasoning trajectories that may support the generation of the target answer. However, the integral is generally intractable because the latent space is continuous and high-dimensional.

To obtain a tractable objective, we introduce a variational posterior $q_\phi(\mathbf{Z}|\mathbf{x},\mathbf{y})$, which approximates the target-conditioned latent distribution during training. Multiplying and dividing by $q_\phi(\mathbf{Z}|\mathbf{x},\mathbf{y})$ inside the integral gives
\begin{equation}
\log p_\theta(\mathbf{y}|\mathbf{x})
=
\log \int q_\phi(\mathbf{Z}|\mathbf{x},\mathbf{y})
\frac{p_\theta(\mathbf{y},\mathbf{Z}|\mathbf{x})}{q_\phi(\mathbf{Z}|\mathbf{x},\mathbf{y})}\, d\mathbf{Z}.
\end{equation}
This reformulation allows us to express the marginal likelihood as an expectation with respect to the variational posterior.

Applying Jensen's inequality yields a lower bound on the log-marginal likelihood:
\begin{equation}
\log p_\theta(\mathbf{y}|\mathbf{x})
\ge
\mathbb{E}_{q_\phi(\mathbf{Z}|\mathbf{x},\mathbf{y})}
\left[
\log \frac{p_\theta(\mathbf{y},\mathbf{Z}|\mathbf{x})}{q_\phi(\mathbf{Z}|\mathbf{x},\mathbf{y})}
\right].
\end{equation}
Using the factorization
\begin{equation}
p_\theta(\mathbf{y},\mathbf{Z}|\mathbf{x})=p_\theta(\mathbf{y}|\mathbf{x},\mathbf{Z})\,p_\theta(\mathbf{Z}|\mathbf{x}),
\end{equation}
we obtain
\begin{align}
\log p_\theta(\mathbf{y}|\mathbf{x})
&\ge
\mathbb{E}_{q_\phi(\mathbf{Z}|\mathbf{x},\mathbf{y})}[\log p_\theta(\mathbf{y}|\mathbf{x},\mathbf{Z})]
-
D_{\mathrm{KL}}\!\left(q_\phi(\mathbf{Z}|\mathbf{x},\mathbf{y})\,\|\,p_\theta(\mathbf{Z}|\mathbf{x})\right).
\end{align}
This is the standard evidence lower bound (ELBO) used in our formulation~\cite{kingma2013auto}.

The ELBO contains two terms with distinct roles. The reconstruction term,
\[
\mathbb{E}_{q_\phi(\mathbf{Z}|\mathbf{x},\mathbf{y})}[\log p_\theta(\mathbf{y}|\mathbf{x},\mathbf{Z})],
\]
encourages the decoder to generate the target answer conditioned on latent reasoning states sampled from the posterior. The KL term,
\[
D_{\mathrm{KL}}\!\left(q_\phi(\mathbf{Z}|\mathbf{x},\mathbf{y})\,\|\,p_\theta(\mathbf{Z}|\mathbf{x})\right),
\]
encourages the target-agnostic prior to match the target-aware posterior. Together, these two terms define the standard variational objective for latent-variable conditional generation.

In our model, the conditional likelihood is parameterized autoregressively:
\begin{equation}
\log p_\theta(\mathbf{y}|\mathbf{x},\mathbf{Z})
=
\sum_t \log p_\theta(y_t \mid \mathbf{x}, \mathbf{y}_{<t}, \mathbf{Z}).
\end{equation}
Accordingly, the reconstruction term in the ELBO is instantiated as the standard next-token prediction objective conditioned on the sampled latent sequence $\mathbf{Z}$. In other words, the decoder learns to generate the answer token by token while treating the latent reasoning sequence as an additional continuous conditioning signal~\cite{bowman2016generating}.

This standard ELBO serves as the variational foundation of our method. However, as discussed in the main text, ELBO optimization alone is insufficient in our setting because the posterior has access to the target during training, whereas inference must rely solely on the prior~\cite{razavi2019preventing, alias2017z}. This motivates the additional reverse-KL regularization introduced in AMVL.

\subsection{Evidence Upper Bound (EUBO) and Reverse KL Regularization}
\label{appendix:eubo}

As discussed above, standard ELBO optimization alone is insufficient in our setting, because it only encourages the prior to match the target-aware posterior, while leaving the posterior itself weakly constrained with respect to inference-time usability~\cite{bowman2016generating, zhao2017towards, he2019lagging}. This motivates introducing an additional reverse-side regularizer on the posterior branch. One theoretical motivation for such a reverse-KL term comes from the Evidence Upper Bound (EUBO)~\cite{hinton1995wake, grosse2015sandwiching} view of variational inference.

While the standard ELBO relies on the forward KL divergence and provides a lower bound on the marginal likelihood, the reverse KL divergence with respect to the true target-conditioned posterior gives rise to an EUBO.
Consider the reverse KL divergence from the true target-conditioned posterior $p_\theta(\mathbf{Z}|\mathbf{x},\mathbf{y})$ to the variational posterior $q_\phi(\mathbf{Z}|\mathbf{x},\mathbf{y})$. By non-negativity of KL divergence,
\begin{equation}
D_{\mathrm{KL}}\!\left(p_\theta(\mathbf{Z}|\mathbf{x},\mathbf{y})\,\|\,q_\phi(\mathbf{Z}|\mathbf{x},\mathbf{y})\right) \ge 0.
\end{equation}
By the definition of conditional probability, we have
\[
p_\theta(\mathbf{Z}|\mathbf{x},\mathbf{y})=\frac{p_\theta(\mathbf{y},\mathbf{Z}|\mathbf{x})}{p_\theta(\mathbf{y}|\mathbf{x})},
\]
we can expand:
\begin{align}
D_{\mathrm{KL}}\!\left(p_\theta(\mathbf{Z}|\mathbf{x},\mathbf{y})\,\|\,q_\phi(\mathbf{Z}|\mathbf{x},\mathbf{y})\right)
&=
\mathbb{E}_{p_\theta(\mathbf{Z}|\mathbf{x},\mathbf{y})}
\left[
\log \frac{p_\theta(\mathbf{Z}|\mathbf{x},\mathbf{y})}{q_\phi(\mathbf{Z}|\mathbf{x},\mathbf{y})}
\right] \nonumber \\
&=
\mathbb{E}_{p_\theta(\mathbf{Z}|\mathbf{x},\mathbf{y})}
\left[
\log \frac{p_\theta(\mathbf{y},\mathbf{Z}|\mathbf{x})}{p_\theta(\mathbf{y}|\mathbf{x})\,q_\phi(\mathbf{Z}|\mathbf{x},\mathbf{y})}
\right] \nonumber \\
&=
\mathbb{E}_{p_\theta(\mathbf{Z}|\mathbf{x},\mathbf{y})}
\left[
\log \frac{p_\theta(\mathbf{y},\mathbf{Z}|\mathbf{x})}{q_\phi(\mathbf{Z}|\mathbf{x},\mathbf{y})}
\right]
-
\log p_\theta(\mathbf{y}|\mathbf{x}).
\end{align}
Rearranging the terms, we obtain an upper bound on the log-marginal likelihood:
\begin{equation}
\log p_\theta(\mathbf{y}|\mathbf{x})
\le
\mathbb{E}_{p_\theta(\mathbf{Z}|\mathbf{x},\mathbf{y})}
\left[
\log \frac{p_\theta(\mathbf{y},\mathbf{Z}|\mathbf{x})}{q_\phi(\mathbf{Z}|\mathbf{x},\mathbf{y})}
\right]
\triangleq
\mathcal{U}_{\mathrm{EUBO}}.
\end{equation}
This establishes the exact EUBO. Minimizing $\mathcal{U}_{\mathrm{EUBO}}$ with respect to the variational parameters $\phi$ is mathematically equivalent to minimizing the exact reverse KL divergence. Unlike the forward KL, this reverse direction penalizes approximate distributions that assign insufficient density to regions supported by the reference distribution. This support-coverage preference~\cite{minka2005divergence, bishop2006pattern} is one reason why reverse-KL-style objectives are often invoked when motivating upper-bound-oriented regularization principles.

\paragraph{From Exact EUBO to Practical Surrogate.}
Directly minimizing the exact reverse KL requires taking expectations under the true posterior $p_\theta(\mathbf{Z}|\mathbf{x},\mathbf{y})$, which is analytically intractable and unavailable during end-to-end optimization. 

In our setting, an ideal reference distribution for posterior regularization should reflect inference-time latent support while remaining tractable. Since the true target-conditioned posterior is not available in closed form, we instead use the learned prior $p_\theta(\mathbf{Z}|\mathbf{x})$ as a tractable reference distribution and define the practical reverse regularizer as:
\begin{equation}
\mathcal{L}_{\mathrm{rev}}
=
D_{\mathrm{KL}}\!\left(\mathrm{sg}[p_\theta(\mathbf{Z}|\mathbf{x})]\,\|\,q_\phi(\mathbf{Z}|\mathbf{x},\mathbf{y})\right),
\end{equation}
where $\mathrm{sg}[\cdot]$ denotes the stop-gradient operator. This design yields a practical reverse-KL regularizer that faithfully inherits the support-compatibility preference of the exact EUBO perspective. Importantly, the stop-gradient ensures that the loss updates only the posterior branch. 

Therefore, the $\mathcal{L}_{\mathrm{rev}}$ term used in the main text is best interpreted as an EUBO-motivated practical surrogate, rather than a direct optimization of the exact theoretical EUBO. Its primary purpose is not to optimize a formal upper bound directly, but to robustly regularize the posterior toward inference-compatible support defined by the learned prior.

\begin{figure}[t]
  \centering
  \includegraphics[width=0.8\columnwidth]{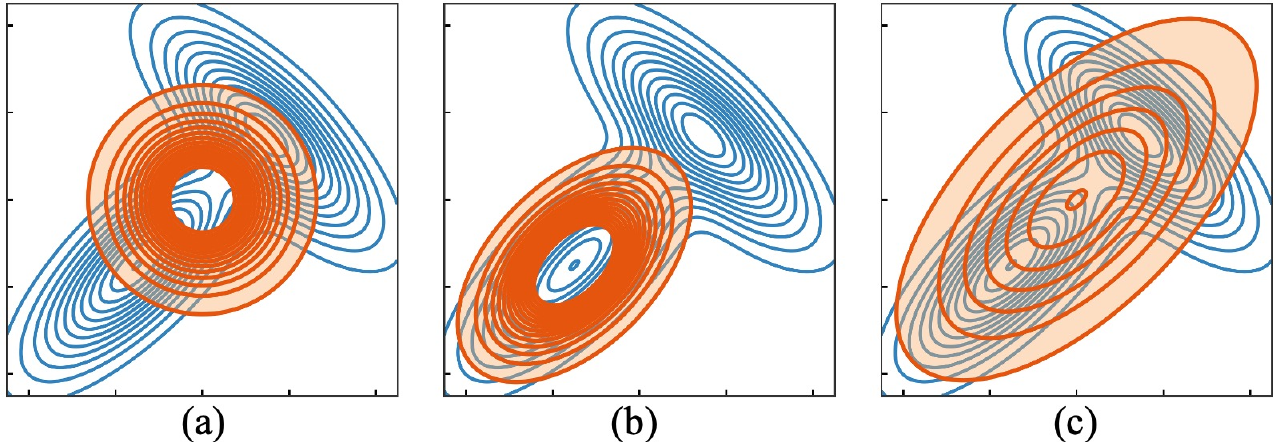}
  \caption{
Geometric illustration of the two KL directions underlying AMVL on a bimodal latent distribution. 
Blue contours denote the reference distribution, and orange contours denote the optimized approximate distribution. 
(a) Initial mismatch. 
(b) Forward KL alignment yields a sharper, mode-seeking fit. 
(c) Reverse KL regularization encourages broader support coverage.
   }
  \label{viusal_elbo_eubo}
\end{figure}

\subsection{Complementary Roles of Forward and Reverse KL}
\label{appendix:complementary_kl}

The two KL directions in AMVL play different and complementary roles:
\begin{align}
\mathcal{L}_{\mathrm{fwd}}
&=
D_{\mathrm{KL}}\!\left(\mathrm{sg}[q_\phi(\mathbf{Z}|\mathbf{x},\mathbf{y})]\,\|\,p_\theta(\mathbf{Z}|\mathbf{x})\right),\\
\mathcal{L}_{\mathrm{rev}}
&=
D_{\mathrm{KL}}\!\left(\mathrm{sg}[p_\theta(\mathbf{Z}|\mathbf{x})]\,\|\,q_\phi(\mathbf{Z}|\mathbf{x},\mathbf{y})\right).
\end{align}
Because of the stop-gradient operators, the two objectives affect different parameter sets. The forward term $\mathcal{L}_{\mathrm{fwd}}$ updates the prior while keeping the posterior fixed, thereby calibrating the target-agnostic prior toward posterior-inferred latent reasoning states. In contrast, the reverse term $\mathcal{L}_{\mathrm{rev}}$ updates the posterior while keeping the prior fixed, thereby discouraging posterior solutions that are excessively sharp, overly target-specific, or poorly supported by the learned prior. 
It is crucial to understand the geometric nature of these KL terms. In both \(\mathcal{L}_{\mathrm{fwd}}\) and \(\mathcal{L}_{\mathrm{rev}}\), the optimization is performed with respect to the \emph{second argument} of the KL divergence (\(p\) in \(\mathcal{L}_{\mathrm{fwd}}\) and \(q\) in \(\mathcal{L}_{\mathrm{rev}}\)). Minimizing \(D_{\mathrm{KL}}(P \,\|\, Q)\) with respect to \(Q\) is known to have a mass-covering effect, forcing the optimized distribution \(Q\) to spread out and cover the support of the fixed distribution \(P\). Therefore, both KL terms in AMVL are mass-covering, not mode-seeking. The complementarity arises from the \emph{asymmetric direction} of this mass-covering behavior:
\begin{itemize}
    \item \textbf{Forward KL (\(\mathcal{L}_{\mathrm{fwd}}\)) forces the prior \(p\) to cover the posterior \(q\).} This encourages the prior to be expressive enough to represent all reasoning states the posterior finds useful during training.
    \item \textbf{Reverse KL (\(\mathcal{L}_{\mathrm{rev}}\)) forces the posterior \(q\) to cover the prior \(p\).} This regularizes the posterior, preventing it from collapsing into a narrow, answer-dependent mode that is unsupported by the prior (i.e., a region where \(p(\mathbf{z})\) is low).
\end{itemize}
In essence, AMVL establishes a mutual mass-covering dynamic in which the prior and posterior are encouraged to cover each other. This differs from the common mode-seeking versus mass-covering trade-off and is the key to calibrating the latent space. The illustration in Figure~\ref{viusal_elbo_eubo}, which contrasts mode-seeking and mass-covering, should be interpreted with this understanding: both KL objectives in AMVL exhibit the mass-covering property illustrated in Figure~\ref{viusal_elbo_eubo}(c), but they are applied in opposing directions between the prior and posterior.

This separation is important in our setting. If only the forward KL is used, the prior is trained to chase the posterior, but the posterior itself is weakly constrained with respect to inference-time usability. If only the reverse term is used, posterior spread may improve, but the prior is not explicitly trained to match posterior-inferred reasoning states. Their combination therefore reduces train-inference mismatch from both directions: the prior is encouraged to approach the posterior, and the posterior is simultaneously discouraged from drifting too far away from what the prior can support at inference time.

\subsection{Diagonal Gaussian Case and Closed-Form Effects}
\label{appendix:gaussian_kl}

In our implementation, both the prior and posterior are factorized diagonal Gaussians~\cite{kingma2013auto, rezende2014stochastic} over the $k$ latent slots:
\begin{equation}
p_\theta(\mathbf{Z}|\mathbf{x})=\prod_{i=1}^k p_\theta(\mathbf{z}_i|\mathbf{x}),
\qquad
q_\phi(\mathbf{Z}|\mathbf{x},\mathbf{y})=\prod_{i=1}^k q_\phi(\mathbf{z}_i|\mathbf{x},\mathbf{y}).
\end{equation}
Therefore, the KL divergence between the full latent sequences decomposes across slots:
\begin{equation}
D_{\mathrm{KL}}(q_\phi(\mathbf{Z}|\mathbf{x},\mathbf{y})\,\|\,p_\theta(\mathbf{Z}|\mathbf{x}))
=
\sum_{i=1}^k
D_{\mathrm{KL}}(q_\phi(\mathbf{z}_i|\mathbf{x},\mathbf{y})\,\|\,p_\theta(\mathbf{z}_i|\mathbf{x})),
\end{equation}
and similarly for the reverse direction.

For two diagonal Gaussian distributions
\[
\mathcal{N}_1=\mathcal{N}(\boldsymbol{\mu}_1,\mathrm{diag}(\boldsymbol{\sigma}_1^2)),
\qquad
\mathcal{N}_2=\mathcal{N}(\boldsymbol{\mu}_2,\mathrm{diag}(\boldsymbol{\sigma}_2^2)),
\]
the KL divergence has the closed form
\begin{equation}
D_{\mathrm{KL}}(\mathcal{N}_1\,\|\,\mathcal{N}_2)
=
\frac{1}{2}
\sum_{j=1}^{d}
\left(
\frac{\sigma_{1,j}^2 + (\mu_{1,j}-\mu_{2,j})^2}{\sigma_{2,j}^2}
-1
+\log\frac{\sigma_{2,j}^2}{\sigma_{1,j}^2}
\right).
\end{equation}

This expression makes the effects of the two KL directions explicit. For the forward KL, large mean mismatch $(\mu_{1,j}-\mu_{2,j})^2$ and insufficient prior variance are penalized, which pushes the prior toward posterior-inferred latent states. For the reverse KL, the posterior is penalized when it assigns too little variance or too little mass to regions supported by the prior, thereby discouraging posterior over-concentration and improving support compatibility.

Since our implementation predicts log-variance, it is useful to rewrite the KL in terms of
\[
\mathbf{L} = \log \boldsymbol{\sigma}^2.
\]
Substituting $L_{1,j}=\log \sigma_{1,j}^2$ and $L_{2,j}=\log \sigma_{2,j}^2$ into the component-wise formula, we obtain:
\begin{equation}
D_{\mathrm{KL}}(\mathcal{N}_1\,\|\,\mathcal{N}_2)
=
\frac{1}{2}
\sum_{j=1}^{d}
\left(
\exp(L_{1,j}-L_{2,j})
+
(\mu_{1,j}-\mu_{2,j})^2 \exp(-L_{2,j})
-1
-(L_{1,j}-L_{2,j})
\right).
\end{equation}
For the reverse direction (where $\mathcal{N}_2$ represents the posterior), when the posterior log-variance $L_{2,j}$ becomes excessively small, both the variance-ratio term and the mean-mismatch term can grow rapidly. Thus, minimizing $D_{\mathrm{KL}}(p_\theta\|q_\phi)$ explicitly penalizes posterior over-concentration relative to the learned prior.

In practice, we compute both KL terms in closed form for each latent slot and then average over latent slots and batch elements.

\section{Theoretical Analysis of Answer Leakage and AMVL}
\label{appendix:theory_amvl_leakage}

We provide a theoretical analysis of why bidirectional KL regularization in AMVL better mitigates answer leakage than standard one-sided ELBO training.

\paragraph{Setup.} Our formulation builds on the standard conditional variational framework~\cite{kingma2013auto,sohn2015learning}.
Let
$p_\theta(\mathbf z \mid \mathbf x)$
denote the inference-time prior, and
$q_\phi(\mathbf z \mid \mathbf x, \mathbf y)
$
the training-time posterior. We consider the common diagonal Gaussian parameterization
\[
q_\phi(\mathbf z \mid \mathbf x, \mathbf y)
= \mathcal N\!\big(\mu_\phi(\mathbf x,\mathbf y), \mathrm{diag}(\sigma_\phi^2(\mathbf x,\mathbf y))\big),
\]
\[
p_\theta(\mathbf z \mid \mathbf x)
= \mathcal N\!\big(\mu_\theta(\mathbf x), \mathrm{diag}(\sigma_\theta^2(\mathbf x))\big).
\]

To formalize answer leakage, we decompose the posterior mean into an input-grounded component and an answer-dependent shift:
\begin{equation}
\mu_\phi(\mathbf x,\mathbf y)
=
\mu_\phi^{(x)}(\mathbf x)
+
\delta(\mathbf x,\mathbf y),
\label{eq:mu_decomp}
\end{equation}
where $\delta(\mathbf x,\mathbf y)$ denotes the additional posterior displacement induced by the target answer $\mathbf y$. When $\delta \equiv 0$, the posterior mean is fully grounded in the input $\mathbf x$, with no answer-dependent mean shift.

We further define the conditional mean leakage
\begin{equation}
\bar{\delta}(\mathbf x)
:=
\mathbb E_{\mathbf y \sim p(\mathbf y \mid \mathbf x)}[\delta(\mathbf x,\mathbf y)].
\label{eq:mean_leakage}
\end{equation}

In this appendix, we focus on \emph{mean-level leakage}, i.e., answer-dependent posterior drift in the latent mean. This restriction is deliberate: mean drift alone already suffices to induce prior contamination under one-sided KL matching, and it admits a clean closed-form analysis of both the contaminated prior optimum and the reverse-KL corrective gradient. We leave leakage through answer-dependent posterior variance or higher-order statistics to future work.

\paragraph{Ideal inference-time latent distribution.}
Let
$
p^*(\mathbf z \mid \mathbf x)
$
denote an ideal inference-time latent distribution that depends only on $\mathbf x$ and supports optimal downstream prediction. In the propositions below, we will explicitly state when we assume that the mean of $p^*$ matches the input-grounded posterior component:
\begin{equation}
\mu^*(\mathbf x) = \mu_\phi^{(x)}(\mathbf x).
\label{eq:ideal_input_grounded}
\end{equation}
This assumption formalizes the desideratum that the ideal test-time latent state should be input-grounded rather than answer-conditioned.

\subsection{Prior Contamination Under ELBO}
\label{appendix:prior_contamination_proof}

The KL term in the standard conditional ELBO is
\begin{equation}
\mathcal L_{\mathrm{ELBO}}^{\mathrm{KL}}
=
\mathbb E_{(\mathbf x,\mathbf y)\sim p_{\mathrm{data}}}
\left[
D_{KL}\big(q_\phi(\mathbf z\mid \mathbf x,\mathbf y)\,\|\,p_\theta(\mathbf z\mid \mathbf x)\big)
\right].
\label{eq:elbo_kl}
\end{equation}

\begin{proposition}[Prior contamination under ELBO]
\label{prop:prior_contamination}
Assume diagonal Gaussian latents with fixed posterior $q_\phi$. Then minimizing the ELBO KL term with respect to the prior mean yields
\begin{equation}
\mu_\theta^*(\mathbf x)
=
\mathbb E_{\mathbf y \sim p(\mathbf y\mid \mathbf x)}
\left[
\mu_\phi(\mathbf x,\mathbf y)
\right]
=
\mu_\phi^{(x)}(\mathbf x)+\bar{\delta}(\mathbf x).
\label{eq:prior_optimum_elbo}
\end{equation}
Consequently, if the ideal inference-time latent mean is input-grounded as in \eqref{eq:ideal_input_grounded}, then one-sided ELBO matching causes the learned prior mean to inherit the average answer-dependent shift as residual bias. In particular, if $p^*(\mathbf z \mid \mathbf x)$ shares the same diagonal covariance as the learned prior, then
\begin{equation}
D_{KL}\!\big(p^*(\mathbf z\mid \mathbf x)\,\|\,p_\theta^*(\mathbf z\mid \mathbf x)\big)
=
\frac{1}{2}
\sum_{j=1}^d
\frac{\bar{\delta}_j(\mathbf x)^2}{\sigma_{\theta,j}^2(\mathbf x)}.
\label{eq:prior_contamination_exact}
\end{equation}
More generally, \eqref{eq:prior_contamination_exact} isolates the contribution due purely to mean contamination.
\end{proposition}

\begin{proof}
For fixed posterior $q_\phi$, minimizing
\[
\mathbb E_{\mathbf y\mid \mathbf x}
\left[
D_{KL}\big(q_\phi(\mathbf z\mid \mathbf x,\mathbf y)\,\|\,p_\theta(\mathbf z\mid \mathbf x)\big)
\right]
\]
with respect to the Gaussian prior mean $\mu_\theta(\mathbf x)$ is equivalent to minimizing, for each $\mathbf x$, the expected quadratic mean-mismatch term
\[
\mathbb E_{\mathbf y\mid \mathbf x}
\left[
\sum_{j=1}^d
\frac{\big(\mu_{\phi,j}(\mathbf x,\mathbf y)-\mu_{\theta,j}(\mathbf x)\big)^2}{2\sigma_{\theta,j}^2(\mathbf x)}
\right].
\]
Taking derivatives with respect to $\mu_{\theta,j}(\mathbf x)$ and setting them to zero yields
\[
\mu_{\theta,j}^*(\mathbf x)
=
\mathbb E_{\mathbf y\mid \mathbf x}\big[\mu_{\phi,j}(\mathbf x,\mathbf y)\big].
\]
Using the decomposition in \eqref{eq:mu_decomp},
\[
\mu_{\phi,j}(\mathbf x,\mathbf y)
=
\mu_{\phi,j}^{(x)}(\mathbf x)+\delta_j(\mathbf x,\mathbf y),
\]
we obtain
\[
\mu_{\theta,j}^*(\mathbf x)
=
\mathbb E_{\mathbf y\mid \mathbf x}
\big[
\mu_{\phi,j}^{(x)}(\mathbf x)+\delta_j(\mathbf x,\mathbf y)
\big]
=
\mu_{\phi,j}^{(x)}(\mathbf x)
+
\mathbb E_{\mathbf y\mid \mathbf x}\big[\delta_j(\mathbf x,\mathbf y)\big]
=
\mu_{\phi,j}^{(x)}(\mathbf x)+\bar\delta_j(\mathbf x).
\]
Stacking coordinates gives
\[
\mu_\theta^*(\mathbf x)
=
\mu_\phi^{(x)}(\mathbf x)+\bar\delta(\mathbf x),
\]
which proves \eqref{eq:prior_optimum_elbo}.

Now assume that the ideal inference-time latent distribution $p^*(\mathbf z\mid \mathbf x)$ has mean
\[
\mu^*(\mathbf x)=\mu_\phi^{(x)}(\mathbf x)
\]
as in \eqref{eq:ideal_input_grounded}, and that $p^*(\mathbf z\mid \mathbf x)$ shares the same diagonal covariance as the learned prior $p_\theta^*(\mathbf z\mid \mathbf x)$, namely $\mathrm{diag}(\sigma_\theta^2(\mathbf x))$. Then the KL divergence between these two diagonal Gaussians is
\[
D_{KL}\!\big(p^*(\mathbf z\mid \mathbf x)\,\|\,p_\theta^*(\mathbf z\mid \mathbf x)\big)
=
\frac{1}{2}
\sum_{j=1}^d
\left(
\frac{\sigma_{\theta,j}^2(\mathbf x)}{\sigma_{\theta,j}^2(\mathbf x)}
+
\frac{\big(\mu_j^*(\mathbf x)-\mu_{\theta,j}^*(\mathbf x)\big)^2}{\sigma_{\theta,j}^2(\mathbf x)}
-
1
+
\log\frac{\sigma_{\theta,j}^2(\mathbf x)}{\sigma_{\theta,j}^2(\mathbf x)}
\right).
\]
Because the two distributions share the same covariance, the variance-ratio and log-determinant terms cancel, leaving
\[
D_{KL}\!\big(p^*(\mathbf z\mid \mathbf x)\,\|\,p_\theta^*(\mathbf z\mid \mathbf x)\big)
=
\frac{1}{2}
\sum_{j=1}^d
\frac{\big(\mu_j^*(\mathbf x)-\mu_{\theta,j}^*(\mathbf x)\big)^2}{\sigma_{\theta,j}^2(\mathbf x)}.
\]
Substituting $\mu^*(\mathbf x)=\mu_\phi^{(x)}(\mathbf x)$ and
\[
\mu_\theta^*(\mathbf x)
=
\mu_\phi^{(x)}(\mathbf x)+\bar\delta(\mathbf x)
\]
gives
\[
\mu_j^*(\mathbf x)-\mu_{\theta,j}^*(\mathbf x)
=
-\bar\delta_j(\mathbf x),
\]
and therefore
\[
D_{KL}\!\big(p^*(\mathbf z\mid \mathbf x)\,\|\,p_\theta^*(\mathbf z\mid \mathbf x)\big)
=
\frac{1}{2}
\sum_{j=1}^d
\frac{\bar\delta_j(\mathbf x)^2}{\sigma_{\theta,j}^2(\mathbf x)},
\]
which is exactly \eqref{eq:prior_contamination_exact}. Under the shared-covariance assumption, this expression isolates the contribution due purely to mean contamination.
\end{proof}

Proposition~\ref{prop:prior_contamination} formalizes the central problem of one-sided ELBO training in our setting: answer-dependent posterior bias is not merely tolerated during training, but is absorbed by the inference-time prior~\cite{hoffman2016elbo, zhao2017infovae}.

\subsection{Why Forward Alignment Alone Is Insufficient}
\label{appendix:fwd_same_optimum_proof}

AMVL uses the forward alignment term
\begin{equation}
\mathcal L_{\mathrm{fwd}}
=
\mathbb E_{(\mathbf x,\mathbf y)\sim p_{\mathrm{data}}}
\left[
D_{KL}\big(\mathrm{sg}[q_\phi(\mathbf z\mid \mathbf x,\mathbf y)] \,\|\, p_\theta(\mathbf z\mid \mathbf x)\big)
\right].
\label{eq:appendix_fwd}
\end{equation}

\begin{proposition}[Forward alignment alone is insufficient]
\label{prop:fwd_insufficient}
For fixed posterior $q_\phi$, minimizing $\mathcal L_{\mathrm{fwd}}$ over $\theta$ yields the same optimal prior mean as minimizing the KL term in the standard ELBO:
\begin{equation}
\mu_\theta^{*,\mathrm{fwd}}(\mathbf x)
=
\mathbb E_{\mathbf y\mid \mathbf x}[\mu_\phi(\mathbf x,\mathbf y)].
\label{eq:fwd_same_optimum}
\end{equation}
Moreover,
\[
\nabla_\phi \mathcal L_{\mathrm{fwd}} = 0.
\]
Therefore, in isolation, forward alignment calibrates the prior to the current posterior but exerts no direct corrective gradient on answer-dependent bias already present in the posterior.
\end{proposition}

\begin{proof}
The stop-gradient operator treats $q_\phi$ as constant with respect to $\phi$, hence
\[
\nabla_\phi \mathcal L_{\mathrm{fwd}} = 0.
\]
With respect to $\theta$, the objective is identical to minimizing the same forward KL over the prior, so the stationary condition is exactly the same as in Proposition~\ref{prop:prior_contamination}, yielding \eqref{eq:fwd_same_optimum}.
\end{proof}

Proposition~\ref{prop:fwd_insufficient} clarifies the role of forward KL: it is necessary for prior calibration, but by itself it cannot eliminate answer leakage at the source.

\subsection{Reverse KL Suppresses Incompatible Posterior Drift}
\label{appendix:reverse_restoring_proof}

AMVL further introduces the reverse regularizer
\begin{equation}
\mathcal L_{\mathrm{rev}}
=
\mathbb E_{(\mathbf x,\mathbf y)\sim p_{\mathrm{data}}}
\left[
D_{KL}\big(\mathrm{sg}[p_\theta(\mathbf z\mid \mathbf x)] \,\|\, q_\phi(\mathbf z\mid \mathbf x,\mathbf y)\big)
\right].
\label{eq:appendix_rev}
\end{equation}

\begin{proposition}[Reverse KL suppresses incompatible posterior drift]
\label{prop:reverse_restoring}
For diagonal Gaussian latents, the reverse KL in \eqref{eq:appendix_rev} has the closed form
\begin{equation}
\mathcal L_{\mathrm{rev}}
=
\frac{1}{2}
\sum_{j=1}^d
\left(
\frac{\sigma_{\theta,j}^2}{\sigma_{\phi,j}^2}
+
\frac{(\mu_{\theta,j}-\mu_{\phi,j})^2}{\sigma_{\phi,j}^2}
-
1
+
\log\frac{\sigma_{\phi,j}^2}{\sigma_{\theta,j}^2}
\right),
\label{eq:reverse_closed_form}
\end{equation}
where $p_\theta$ is treated as constant by stop-gradient. Its gradient with respect to the posterior mean is
\begin{equation}
\frac{\partial \mathcal L_{\mathrm{rev}}}{\partial \mu_{\phi,j}}
=
\frac{\mu_{\phi,j}-\mu_{\theta,j}}{\sigma_{\phi,j}^2}.
\label{eq:reverse_mu_grad}
\end{equation}
Thus, reverse KL penalizes posterior drift away from prior-compatible high-density latent regions, with a mean-mismatch penalty that becomes stronger when the posterior becomes sharper.
\end{proposition}

\begin{proof}
The expression in \eqref{eq:reverse_closed_form} is the standard KL divergence between two diagonal Gaussians, with the first argument fixed by stop-gradient. Differentiating the quadratic mean term yields \eqref{eq:reverse_mu_grad}.

To characterize the effect of this gradient, set \eqref{eq:reverse_mu_grad} to zero:
\[
\frac{\mu_{\phi,j}-\mu_{\theta,j}}{\sigma_{\phi,j}^2}=0
\quad\Longrightarrow\quad
\mu_{\phi,j}=\mu_{\theta,j}.
\]
Since $\sigma_{\phi,j}^2>0$, the objective is strictly convex in $\mu_{\phi,j}$, with
\[
\frac{\partial^2 \mathcal L_{\mathrm{rev}}}{\partial \mu_{\phi,j}^2}
=
\frac{1}{\sigma_{\phi,j}^2}>0,
\]
so this stationary point is unique and is the global minimizer with respect to $\mu_{\phi,j}$.

Moreover, the gradient has the same sign as the deviation $(\mu_{\phi,j}-\mu_{\theta,j})$:
\[
\frac{\partial \mathcal L_{\mathrm{rev}}}{\partial \mu_{\phi,j}}
\begin{cases}
>0, & \mu_{\phi,j}>\mu_{\theta,j},\\[4pt]
<0, & \mu_{\phi,j}<\mu_{\theta,j}.
\end{cases}
\]
Therefore, gradient descent on $\mathcal L_{\mathrm{rev}}$ exerts a restoring force that always pulls $\mu_{\phi,j}$ toward $\mu_{\theta,j}$. Its magnitude is
\[
\left|
\frac{\partial \mathcal L_{\mathrm{rev}}}{\partial \mu_{\phi,j}}
\right|
=
\frac{|\mu_{\phi,j}-\mu_{\theta,j}|}{\sigma_{\phi,j}^2},
\]
which increases as $\sigma_{\phi,j}^2$ decreases. Thus, the penalty against posterior drift becomes stronger when the posterior becomes sharper.

Substituting the decomposition in \eqref{eq:mu_decomp} into \eqref{eq:reverse_mu_grad} gives
\begin{equation}
\frac{\partial \mathcal L_{\mathrm{rev}}}{\partial \delta_j}
=
\frac{\mu_{\phi,j}^{(x)}(\mathbf x)+\delta_j(\mathbf x,\mathbf y)-\mu_{\theta,j}(\mathbf x)}{\sigma_{\phi,j}^2(\mathbf x,\mathbf y)}.
\label{eq:reverse_delta_grad}
\end{equation}
Define the prior-centering bias
\[
b_j(\mathbf x):=\mu_{\theta,j}(\mathbf x)-\mu_{\phi,j}^{(x)}(\mathbf x).
\]
Then \eqref{eq:reverse_delta_grad} can be rewritten as
\[
\frac{\partial \mathcal L_{\mathrm{rev}}}{\partial \delta_j}
=
\frac{\delta_j(\mathbf x,\mathbf y)-b_j(\mathbf x)}{\sigma_{\phi,j}^2(\mathbf x,\mathbf y)}.
\]
Hence, reverse KL penalizes answer-dependent posterior drift relative to the current prior center. In the low-contamination regime where $b_j(\mathbf x)\approx 0$---for example, after bidirectional calibration has partially aligned the prior with the input-grounded component---the gradient simplifies to
\begin{equation}
\frac{\partial \mathcal L_{\mathrm{rev}}}{\partial \delta_j}
\approx
\frac{\delta_j(\mathbf x,\mathbf y)}{\sigma_{\phi,j}^2(\mathbf x,\mathbf y)},
\label{eq:reverse_delta_grad_simplified}
\end{equation}
which directly suppresses answer-dependent posterior shift. This effect is strongest precisely when the posterior becomes sharply concentrated.
\end{proof}

\subsection{Bidirectional Calibration Reduces Prior Contamination}
\label{appendix:amvl_shrinkage_proof}

We now formalize the local shrinkage effect of bidirectional calibration under a simplified linear-response model.

Recall that $\gamma$ is the scalar weight controlling the strength of the reverse KL regularization term $\mathcal L_{\mathrm{rev}}$ in the AMVL training objective (Eq.~29). A larger $\gamma$ imposes a stronger restoring force against answer-dependent posterior drift, as formalized below.

\begin{assumption}[Local linear leakage model]
\label{assump:linear_leakage}
In a local neighborhood of training, the posterior mean admits the form
\begin{equation}
\mu_\phi(\mathbf x,\mathbf y)
=
\mu^{(x)}(\mathbf x)
+
\alpha f(\mathbf x,\mathbf y),
\label{eq:linear_leakage_model}
\end{equation}
where $\alpha \ge 0$ denotes a scalar leakage coefficient and $f(\mathbf x,\mathbf y)$ is an answer-dependent direction. Let
\[
\bar f(\mathbf x)
:=
\mathbb E_{\mathbf y \sim p(\mathbf y \mid \mathbf x)}
\big[f(\mathbf x,\mathbf y)\big].
\]
\end{assumption}

Assumption~\ref{assump:linear_leakage} is a local linearization of the general leakage decomposition in Eq.~\eqref{eq:mu_decomp}, which writes
\[
\mu_\phi(\mathbf x,\mathbf y)
=
\mu_\phi^{(x)}(\mathbf x)
+
\delta(\mathbf x,\mathbf y)
\]
for a free vector-valued answer-dependent shift $\delta(\mathbf x,\mathbf y)\in\mathbb R^d$. Here we refine that decomposition by factoring
\[
\delta(\mathbf x,\mathbf y)=\alpha f(\mathbf x,\mathbf y),
\]
where $\alpha$ captures a shared local leakage amplitude and $f(\mathbf x,\mathbf y)$ captures the direction and sample-dependent shape of leakage. Freezing the local shape $f$ and analyzing the scalar coordinate $\alpha$ is what makes the equilibrium analysis tractable. Under this parameterization, the conditional average leakage direction $\bar f(\mathbf x)$ plays the same role as $\bar\delta(\mathbf x)$ in Eq.~\eqref{eq:mean_leakage}, with the correspondence
\[
\bar\delta(\mathbf x)\leftrightarrow \alpha\,\bar f(\mathbf x).
\]

\begin{assumption}[Local linear-response stationary condition]
\label{assump:linear_response}
Near a local stationary point, the training dynamics along the leakage coordinate $\alpha$ admit a first-order linear approximation. In particular, motivated by Proposition~\ref{prop:reverse_restoring}, the reverse KL contributes a restoring force along the leakage coordinate that is locally proportional to
\[
\frac{\gamma\alpha}{\sigma_{\mathrm{eff}}^2},
\]
where $\sigma_{\mathrm{eff}}^2$ denotes an effective posterior variance along the leakage direction. We further assume that, in the local regime of interest, the prior mean is approximately centered on the input-grounded component, i.e.,
\[
\mu_\theta(\mathbf x)\approx \mu^{(x)}(\mathbf x),
\]
so that the reverse-KL contribution along the leakage coordinate is first-order proportional to $\alpha$. Let $\alpha_{\mathrm{ELBO}}$ denote the local equilibrium leakage coefficient under one-sided ELBO training.
\end{assumption}

Here, ``equilibrium'' refers to the stationary point of the leakage-coordinate dynamics during training, i.e., the value of $\alpha$ at which the net gradient force acting on the leakage coefficient is zero. Under one-sided ELBO training, the reconstruction term may favor answer-dependent shortcuts in the posterior, while the forward KL partially resists them; $\alpha_{\mathrm{ELBO}}$ is the local balance point of these effects. Under AMVL, the reverse KL introduces an additional restoring force that pushes $\alpha$ toward zero, thereby shifting the equilibrium to a smaller value.

\begin{proposition}[Bidirectional calibration reduces prior contamination]
\label{prop:amvl_shrinkage}
Under Assumptions~\ref{assump:linear_leakage} and~\ref{assump:linear_response}, there exists a constant $c>0$ such that the local equilibrium leakage coefficient under AMVL satisfies
\begin{equation}
\alpha_{\mathrm{AMVL}}
=
\frac{\alpha_{\mathrm{ELBO}}}{1+c\gamma/\sigma_{\mathrm{eff}}^2}.
\label{eq:alpha_shrinkage}
\end{equation}
Hence, $\alpha_{\mathrm{AMVL}}$ is monotonically decreasing in $\gamma$. Under exact forward matching at the function optimum, the induced prior contamination satisfies
\begin{equation}
\Delta_{\mathrm{AMVL}}(\mathbf x)
=
\|\mu_\theta^{\mathrm{AMVL}}(\mathbf x)-\mu^{(x)}(\mathbf x)\|_2
=
\alpha_{\mathrm{AMVL}}\|\bar f(\mathbf x)\|_2,
\label{eq:amvl_contamination}
\end{equation}
whereas the corresponding contamination under one-sided ELBO matching is
\begin{equation}
\Delta_{\mathrm{ELBO}}(\mathbf x)
=
\alpha_{\mathrm{ELBO}}\|\bar f(\mathbf x)\|_2.
\label{eq:elbo_contamination}
\end{equation}
Therefore,
\begin{equation}
\Delta_{\mathrm{AMVL}}(\mathbf x)
<
\Delta_{\mathrm{ELBO}}(\mathbf x)
\qquad
\text{whenever } \gamma>0 \text{ and } \|\bar f(\mathbf x)\|_2>0.
\label{eq:strict_improvement_local}
\end{equation}
Equality holds only when $\gamma=0$ or $\|\bar f(\mathbf x)\|_2=0$.
\end{proposition}

\begin{proof}
Under Assumption~\ref{assump:linear_response}, the net ELBO-only gradient force along the leakage coordinate is locally linear near the stationary point. Denoting the local curvature constant by $c_1>0$, we write
\[
G_{\mathrm{ELBO}}(\alpha)
\approx
c_1(\alpha_{\mathrm{ELBO}}-\alpha),
\]
whose unique zero is $\alpha_{\mathrm{ELBO}}$, the local equilibrium under one-sided ELBO training.

Next, substitute the decomposition
\[
\mu_\phi(\mathbf x,\mathbf y)
=
\mu^{(x)}(\mathbf x)+\alpha f(\mathbf x,\mathbf y)
\]
into the reverse-KL mean gradient from Proposition~\ref{prop:reverse_restoring}. For each latent dimension $j$,
\[
\frac{\partial \mathcal L_{\mathrm{rev}}}{\partial \mu_{\phi,j}}
=
\frac{\mu_{\phi,j}-\mu_{\theta,j}}{\sigma_{\phi,j}^2}.
\]
Under the local centering assumption $\mu_\theta(\mathbf x)\approx \mu^{(x)}(\mathbf x)$, the mismatch term becomes first-order proportional to $\alpha f_j(\mathbf x,\mathbf y)$. Projecting this gradient onto the leakage coordinate, averaging over latent dimensions and data samples, and absorbing the local geometric and averaging factors into a constant $c_2>0$, the reverse KL contributes a restoring force of the form
\[
G_{\mathrm{rev}}(\alpha)
=
-\frac{c_2\gamma}{\sigma_{\mathrm{eff}}^2}\alpha.
\]

Under AMVL, both forces act simultaneously, so the local equilibrium $\alpha_{\mathrm{AMVL}}$ satisfies
\[
G_{\mathrm{ELBO}}(\alpha)+G_{\mathrm{rev}}(\alpha)=0.
\]
Substituting the two expressions gives
\[
c_1(\alpha_{\mathrm{ELBO}}-\alpha)
-
\frac{c_2\gamma}{\sigma_{\mathrm{eff}}^2}\alpha
=
0,
\]
which implies
\[
c_1\alpha_{\mathrm{ELBO}}
=
\alpha\left(c_1+\frac{c_2\gamma}{\sigma_{\mathrm{eff}}^2}\right).
\]
Solving for $\alpha$ and defining $c:=c_2/c_1>0$ yields
\[
\alpha_{\mathrm{AMVL}}
=
\frac{\alpha_{\mathrm{ELBO}}}{1+c\gamma/\sigma_{\mathrm{eff}}^2},
\]
which is Eq.~\eqref{eq:alpha_shrinkage}.

Monotonicity in $\gamma$ follows by differentiation:
\[
\frac{\partial \alpha_{\mathrm{AMVL}}}{\partial \gamma}
=
-
\frac{c\,\alpha_{\mathrm{ELBO}}/\sigma_{\mathrm{eff}}^2}
{(1+c\gamma/\sigma_{\mathrm{eff}}^2)^2}
<0.
\]

Under exact forward matching at the function optimum, Proposition~\ref{prop:fwd_insufficient} implies that the prior mean matches the conditional average of the regularized posterior mean:
\[
\mu_\theta^{\mathrm{AMVL}}(\mathbf x)
=
\mathbb E_{\mathbf y\mid \mathbf x}[\mu_\phi(\mathbf x,\mathbf y)].
\]
Using Assumption~\ref{assump:linear_leakage},
\[
\mu_\theta^{\mathrm{AMVL}}(\mathbf x)
=
\mu^{(x)}(\mathbf x)
+
\alpha_{\mathrm{AMVL}}\,
\mathbb E_{\mathbf y\mid \mathbf x}[f(\mathbf x,\mathbf y)]
=
\mu^{(x)}(\mathbf x)+\alpha_{\mathrm{AMVL}}\bar f(\mathbf x),
\]
which gives Eq.~\eqref{eq:amvl_contamination}. The ELBO counterpart in Eq.~\eqref{eq:elbo_contamination} follows identically with $\alpha_{\mathrm{ELBO}}$ in place of $\alpha_{\mathrm{AMVL}}$.

Finally, for $\gamma>0$,
\[
\frac{1}{1+c\gamma/\sigma_{\mathrm{eff}}^2}<1.
\]
Thus, if $\|\bar f(\mathbf x)\|_2>0$, then
\[
\Delta_{\mathrm{AMVL}}(\mathbf x)
=
\frac{\alpha_{\mathrm{ELBO}}\|\bar f(\mathbf x)\|_2}{1+c\gamma/\sigma_{\mathrm{eff}}^2}
<
\alpha_{\mathrm{ELBO}}\|\bar f(\mathbf x)\|_2
=
\Delta_{\mathrm{ELBO}}(\mathbf x),
\]
proving Eq.~\eqref{eq:strict_improvement_local}. Equality occurs only when $\gamma=0$ or $\|\bar f(\mathbf x)\|_2=0$.
\end{proof}

Proposition~\ref{prop:amvl_shrinkage} formalizes the key local mechanism of AMVL: reverse KL reduces the equilibrium answer-dependent posterior shift, and the forward alignment term then transfers this leakage-reduced posterior signal to the prior. Note that if $\bar f(\mathbf x)=0$, the prior mean may exhibit no average contamination even though answer-dependent posterior variability remains; the present result concerns mean-level prior contamination induced by nonzero average leakage directions.

\begin{corollary}[Monotonic improvement with stronger reverse regularization]
\label{cor:gamma_monotone}
Under the assumptions of Proposition~\ref{prop:amvl_shrinkage},
\begin{equation}
\Delta_{\mathrm{AMVL}}(\mathbf x)
=
\frac{\alpha_{\mathrm{ELBO}} \|\bar f(\mathbf x)\|_2}{1+c\gamma/\sigma_{\mathrm{eff}}^2}
\end{equation}
is strictly decreasing in $\gamma$ whenever $\|\bar f(\mathbf x)\|_2>0$.
\end{corollary}

\begin{proof}
Differentiate with respect to $\gamma$:
\[
\frac{\partial \Delta_{\mathrm{AMVL}}}{\partial \gamma}
=
-\frac{c\,\alpha_{\mathrm{ELBO}} \|\bar f(\mathbf x)\|_2/\sigma_{\mathrm{eff}}^2}{(1+c\gamma/\sigma_{\mathrm{eff}}^2)^2}
<0
\qquad
\text{if } \|\bar f(\mathbf x)\|_2>0.
\]
\end{proof}

\begin{corollary}[Stronger benefit under posterior overconfidence]
\label{cor:variance_amplification}
Under the same assumptions, the relative improvement factor
\begin{equation}
\frac{\Delta_{\mathrm{ELBO}}(\mathbf x)}{\Delta_{\mathrm{AMVL}}(\mathbf x)}
=
1+\frac{c\gamma}{\sigma_{\mathrm{eff}}^2}
\end{equation}
increases as $\sigma_{\mathrm{eff}}^2$ decreases.
\end{corollary}

\begin{proof}
Immediate from Proposition~\ref{prop:amvl_shrinkage}.
\end{proof}

Corollary~\ref{cor:variance_amplification} is practically important: within this local Gaussian analysis, reverse KL is most beneficial exactly when answer leakage is most dangerous, namely when the posterior becomes overly sharp and overconfident.

\subsection{Information-Theoretic Interpretation}
\label{appendix:mi_interpretation}

Answer leakage can also be understood through the conditional mutual information $I_q(\mathbf Z;\mathbf Y \mid \mathbf X)$, extending standard ELBO surgery techniques \cite{hoffman2016elbo, alemi2018fixing} to the conditional setting:
\begin{equation}
I_q(\mathbf Z;\mathbf Y \mid \mathbf X)
=
\mathbb E_{p(\mathbf x,\mathbf y)}
\left[
D_{KL}\big(q_\phi(\mathbf z\mid \mathbf x,\mathbf y)\,\|\,q_\phi(\mathbf z\mid \mathbf x)\big)
\right],
\label{eq:cmi_leakage}
\end{equation}
where
\[
q_\phi(\mathbf z\mid \mathbf x)
=
\int q_\phi(\mathbf z\mid \mathbf x,\mathbf y)\,p(\mathbf y\mid \mathbf x)\,d\mathbf y.
\]

\begin{proposition}[Forward KL decomposition]
\label{prop:forward_kl_decomposition}
For any posterior $q_\phi(\mathbf z\mid \mathbf x,\mathbf y)$ and prior $p_\theta(\mathbf z\mid \mathbf x)$,
\begin{equation}
\mathbb E_{p(\mathbf x,\mathbf y)}
D_{KL}\!\big(q_\phi(\mathbf z\mid \mathbf x,\mathbf y)\,\|\,p_\theta(\mathbf z\mid \mathbf x)\big)
=
I_q(\mathbf Z;\mathbf Y\mid \mathbf X)
+
\mathbb E_{p(\mathbf x)}
D_{KL}\!\big(q_\phi(\mathbf z\mid \mathbf x)\,\|\,p_\theta(\mathbf z\mid \mathbf x)\big).
\label{eq:forward_kl_decomposition}
\end{equation}
\end{proposition}

\begin{proof}
Expand the KL term:
\[
D_{KL}(q_\phi(\mathbf z\mid \mathbf x,\mathbf y)\|p_\theta(\mathbf z\mid \mathbf x))
=
D_{KL}(q_\phi(\mathbf z\mid \mathbf x,\mathbf y)\|q_\phi(\mathbf z\mid \mathbf x))
+
\mathbb E_{q_\phi}\!\left[\log \frac{q_\phi(\mathbf z\mid \mathbf x)}{p_\theta(\mathbf z\mid \mathbf x)}\right].
\]
Taking expectation over $p(\mathbf x,\mathbf y)$ yields \eqref{eq:forward_kl_decomposition}.
\end{proof}

A large value of $I_q(\mathbf Z;\mathbf Y \mid \mathbf X)$ indicates that the latent variable retains substantial target-specific information even after conditioning on the input, which is precisely the signature of answer leakage. Proposition~\ref{prop:forward_kl_decomposition} shows that the one-sided KL in the ELBO mixes two effects: suppressing target dependence and fitting the prior to the aggregated posterior. In contrast, AMVL additionally constrains the reverse discrepancy $D_{KL}(p_\theta\|q_\phi)$, forcing each posterior $q_\phi(\mathbf z\mid \mathbf x,\mathbf y)$ to remain compatible with prior-reachable high-density latent regions. This suppresses answer-specific posterior collapse onto narrow latent modes that cannot be reliably reached from $\mathbf x$ alone, thereby helping reduce train-inference mismatch.

\paragraph{Summary.}
The standard ELBO performs one-sided prior matching: it encourages the prior to chase a target-aware posterior, even when that posterior exploits answer leakage. AMVL decouples the two roles. The forward KL teaches the prior to approximate useful training-time latent states, while the reverse KL regularizes the posterior to remain compatible with the inference-time prior. Under the Gaussian mean-leakage analysis above, and under the local linear-response model in Assumptions~\ref{assump:linear_leakage}--\ref{assump:linear_response}, bidirectional calibration reduces the contamination transferred into the prior and therefore provides a principled mechanism for mitigating answer leakage more effectively than one-sided ELBO training.

\section{Baselines}
\label{app:baseline}
We compare our method against representative state-of-the-art MLLM baselines from three families: \emph{thinking about images}, \emph{thinking with images}, and \emph{latent reasoning}.

\textbf{Thinking about Images.}
Methods in this category enhance multimodal reasoning by generating explicit chain-of-thought trajectories over visual inputs, typically through reinforcement learning or supervised reasoning traces.
\begin{itemize}
    \item \textbf{Vision-R1 \cite{visionr1}:} Adapts reinforcement learning to multimodal reasoning by encouraging a ``think before answer'' behavior. The model is optimized to produce explicit reasoning trajectories prior to final response generation, thereby improving step-by-step visual reasoning ability.
    \item \textbf{PAPO \cite{papo}:} Extends reinforcement-learning-based multimodal reasoning with an additional perception-oriented objective. Besides learning to generate reasoning traces, PAPO incorporates an implicit perception loss to encourage more faithful image-grounded descriptions and stronger alignment between visual understanding and reasoning.
\end{itemize}

\textbf{Thinking with Images.}
This category augments reasoning by actively manipulating, refining, or querying visual evidence during inference, often through external tools or environment interaction.
\begin{itemize}
    \item \textbf{PixelReasoner \cite{pixelreasoner}:} Introduces a tool-augmented reasoning framework in which the model iteratively edits or enhances the input image during the reasoning process. By interacting with modified visual observations rather than relying solely on the original input, PixelReasoner improves its ability to resolve fine-grained visual ambiguities and difficult perceptual details.
    \item \textbf{DeepEyes \cite{deepeyes}:} Integrates external tool usage directly into a unified reinforcement learning loop, allowing the model to reason through actions such as visual grounding, web search, and code execution. Instead of treating tools as a separate post-processing module, DeepEyes makes tool invocation part of the reasoning policy itself, aiming to more closely mimic human-like visual perception and problem solving.
\end{itemize}

\textbf{Latent Reasoning.}
This line of work explores reasoning directly in the embedding space, bypassing explicit natural-language or pixel-level intermediate decoding and instead using continuous latent representations as internal reasoning states.
\begin{itemize}
    \item \textbf{LVR \cite{lvr} (Latent Visual Reasoning):} Projects visual features into a joint semantic space and performs autoregressive reasoning by reconstructing query-relevant visual tokens, referred to as ``latent visual thoughts,'' interleaved with text generation. This allows the model to reason over continuous visual semantics without fully materializing every intermediate step in language.
    \item \textbf{Mull-Tokens \cite{mull}:} Introduces modality-agnostic latent tokens that act as a multimodal scratchpad. These tokens are trained with interleaved traces and relaxed supervision so that they can flexibly store intermediate visual or textual information and optimize latent trajectories toward the final answer.
    \item \textbf{Monet \cite{monet}:} Enables reasoning directly in latent visual space by treating continuous embeddings as intermediate visual thoughts. It combines a distillation-based supervised fine-tuning pipeline with Visual-latent Policy Optimization (VLPO), explicitly incorporating latent embeddings into policy-gradient-based optimization to improve generalization on abstract visual reasoning tasks.
\end{itemize}

\section{Implementation Details}
\label{appendix:experimental_details}

\paragraph{Base model and processor.}
All experiments are built upon Qwen2.5-VL-7B-Instruct as the underlying multimodal large language model. We use the corresponding official processor and tokenizer for text-image formatting and multimodal input construction. During training, we extend the tokenizer with a set of latent-specific special tokens, including a latent start token, a latent end token, and a sequence of latent placeholder tokens used to reserve latent slots in the autoregressive input stream.

\paragraph{Training data.}
\label{app:data_mixture}

To construct a robust and diverse environment for continuous latent reasoning, we train AMVL on a comprehensive mixture of multimodal reasoning datasets. This mixture is explicitly curated to cover a wide spectrum of cognitive tasks, from fine-grained visual grounding to complex multi-step logical deduction, ensuring the learned latent space generalizes across varied reasoning paradigms. The composition includes:

\begin{itemize}
    \item \textbf{Visual-CoT~\cite{Visualcot}:} This dataset provides a foundational corpus for step-by-step multimodal reasoning. It trains the model to decompose complex visual questions into intermediate logical steps, bridging high-level semantic queries with observable visual evidence.
    
    \item \textbf{ReFocus~\cite{ReFocus}:} To ensure the latent reasoning space captures precise spatial and region-level information, we include ReFocus. This dataset emphasizes grounded visual reasoning, requiring the model to maintain focus on fine-grained visual details and specific regions of interest throughout the inference process.
    
    \item \textbf{CogCoM~\cite{CogCoM}:} Focusing on complex cognitive trajectories, CogCoM enhances the model's ability to perform multi-hop deductive reasoning. It provides intricate scenarios where the model must synthesize multiple pieces of visual and textual information to arrive at a valid conclusion.
    
    \item \textbf{Zebra-CoT~\cite{Zebra-CoT}:} This dataset is utilized to strengthen relational and compositional reasoning capabilities. It forces the model to track intricate relationships between multiple objects or concepts within a visual scene, discouraging reliance on superficial data correlations.
\end{itemize}
All training samples across these diverse sources are standardized into a unified chat-style multimodal format. During preprocessing, any sample lacking a valid and complete assistant response is strictly filtered out to maintain high supervision quality. For every valid sample, we programmatically insert a dedicated latent token block immediately following the assistant's prefix. Training the main AMVL model is conducted on 16 NVIDIA A100 GPUs and takes approximately 20 hours.

\paragraph{Latent Block Construction and Insertion.}
For every valid sample, we programmatically insert a dedicated latent token block immediately following the assistant's prefix. Given a predefined number of latent slots $k$, this block allocates $k$ dedicated latent placeholder tokens enclosed by specific start and end markers. In our implementation, the latent token block takes the form:
\[
\texttt{<abs\_token>} \;+\; \texttt{latent\_pad}_1,\ldots,\texttt{latent\_pad}_k \;+\; \texttt{</abs\_token>}.
\]
These placeholders serve solely as structural anchors for latent injection and are explicitly excluded from standard language modeling supervision. The start and end markers preserve the latent span boundaries, while the interior placeholder embeddings are replaced by the inferred continuous latent features during both training and inference. This structural modification is critical: it explicitly allocates the continuous capacity required for our dual-KL optimization. By forcing the model to process this latent block before generating any discrete textual output, we ensure that the core abstract and spatial reasoning fluidly evolves within the high-density continuous space prior to textual serialization.

\paragraph{Variational parameterization and latent settings.}
Both the prior and posterior are parameterized as factorized diagonal Gaussian distributions over the latent slots. The latent dimension is set to $d=512$ by default, and the number of latent slots is set to $k=8$ unless otherwise specified in the ablation studies. The variational module predicts the mean and log-variance of each latent slot, and latent samples are drawn using the standard reparameterization trick. In implementation, we predict log-variance rather than variance directly for improved numerical stability.

\paragraph{Input construction and loss masking.}
Training uses teacher-forced autoregressive decoding. The language modeling loss is applied only to the answer tokens following the assistant prefix, while the latent start token, latent end token, and all latent placeholder tokens are masked out from next-token supervision. The latent-token mask is maintained separately so that the model can identify the positions where latent variables should be injected. This design ensures that the model is not trained to predict the fixed latent placeholders themselves, but instead uses them as containers for continuous latent reasoning states.

\paragraph{Optimization details.}
We train the model with AdamW using fused PyTorch implementation, bf16 mixed precision, gradient checkpointing, cosine learning-rate scheduling, and a warmup ratio of $0.05$. The vision encoder is frozen throughout training, while the language backbone and variational modules are jointly optimized. We also enable TF32 matrix multiplication where available for training acceleration. The exact batch size, gradient accumulation steps, learning rate, and training epochs follow the settings reported in the main experiments.

\paragraph{KL scheduling.}
We apply different schedules to the forward and reverse KL terms for stable optimization. The forward KL weight is linearly warmed up from $0$ to $1.0$ over the first 2000 training steps:
\[
\beta_t = \min\left(1, \frac{t}{2000}\right).
\]
For the reverse KL term, we use a delayed annealing strategy. Its weight is fixed to $0$ during the first 1000 steps, and then linearly increased to $0.5$ over the following 2000 steps:
\[
\gamma_t =
\begin{cases}
0, & t < 1000,\\
0.5 \cdot \min\left(1, \frac{t-1000}{2000}\right), & t \ge 1000.
\end{cases}
\]
This schedule allows the prior to become partially calibrated before reverse-side posterior regularization becomes active, which improves training stability in practice.

\paragraph{Embedding initialization and freezing strategy.}
After extending the tokenizer with latent-related special tokens, we resize the model vocabulary accordingly. Newly added token embeddings are initialized using the mean of the original embedding matrix. During the initial embedding adaptation stage, gradient updates to the embedding matrix are restricted to the newly added tokens, preventing unnecessary drift in the pretrained vocabulary. In the main training stage, the language backbone is jointly optimized with the variational components, while the visual encoder remains frozen.

\begin{table}[t]
\caption{Performance on the out-of-distribution VisualPuzzles benchmark.
}
\label{tab:ood}
\centering
\resizebox{0.99\textwidth}{!}{
\begin{tabular}{lccccccc}
\toprule
Method & Overall & Algorithmic &Analogical& Deductive &Inductive &Spatial \\
\midrule
Qwen2.5-VL-7B & 32.71& 37.02& 21.80 &47.50 &26.32& 21.80 \\
Pangea-7B & 31.30& 32.40& 23.70& 38.50& 28.70& 32.50\\
Deepeyes &  32.96 &37.79& 27.01 &41.00& 26.79& 27.01 \\
LVR &27.74&28.63&23.22&36.00&28.23&24.12 \\
LLaVA-OneVision-72B& 30.80& 34.70& 26.50& 37.00& 27.30& 28.70\\
\midrule
Ours-7B & 33.90&32.44&27.96&52.50&28.23&30.77\\
\bottomrule
\end{tabular}
}
\vspace{-7pt}
\end{table}

\begin{figure*}[ht]
    \centering
    \includegraphics[width=0.95\textwidth]{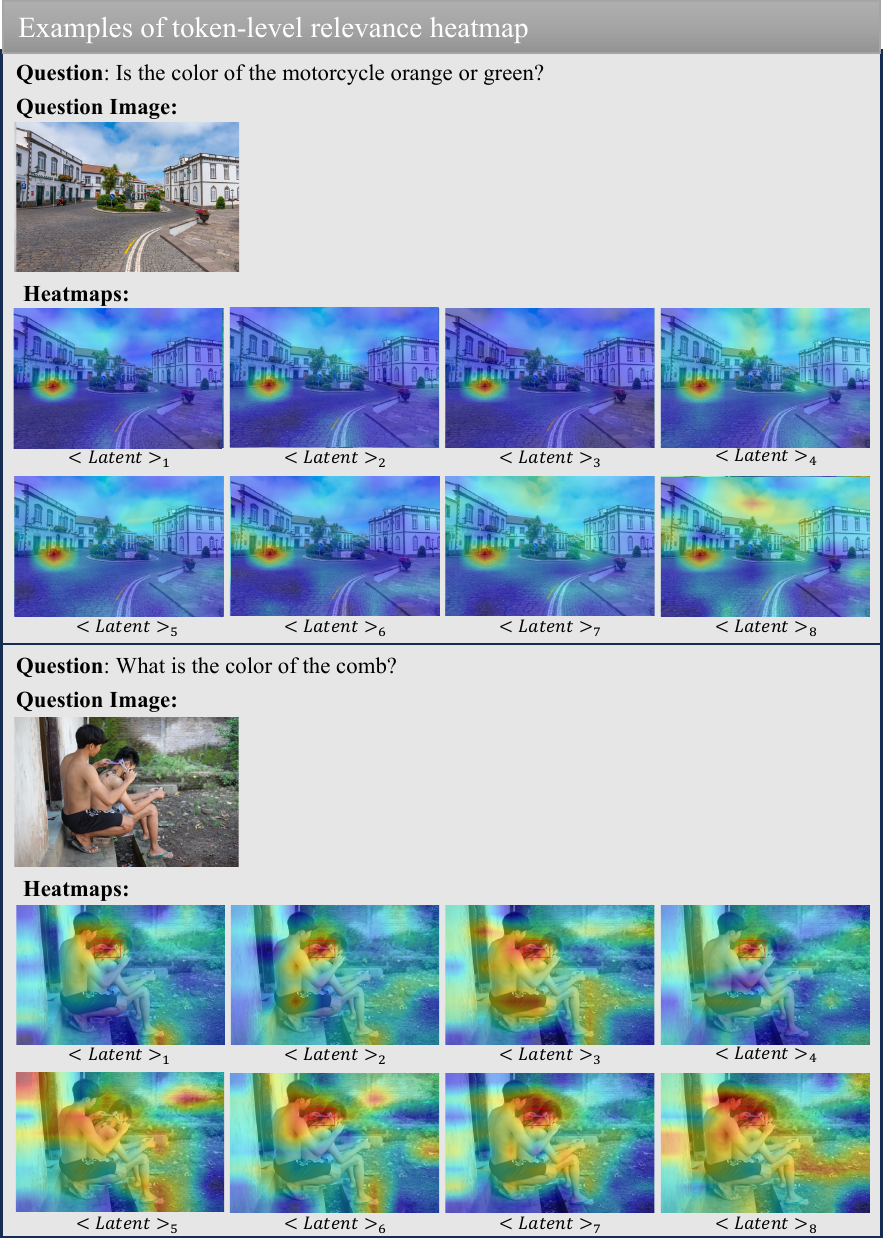}
    \caption{
    Token-level relevance heatmaps.
    }
    \label{fig:heatmap}
\end{figure*}

\begin{figure*}[ht]
    \centering
    \includegraphics[width=0.99\textwidth]{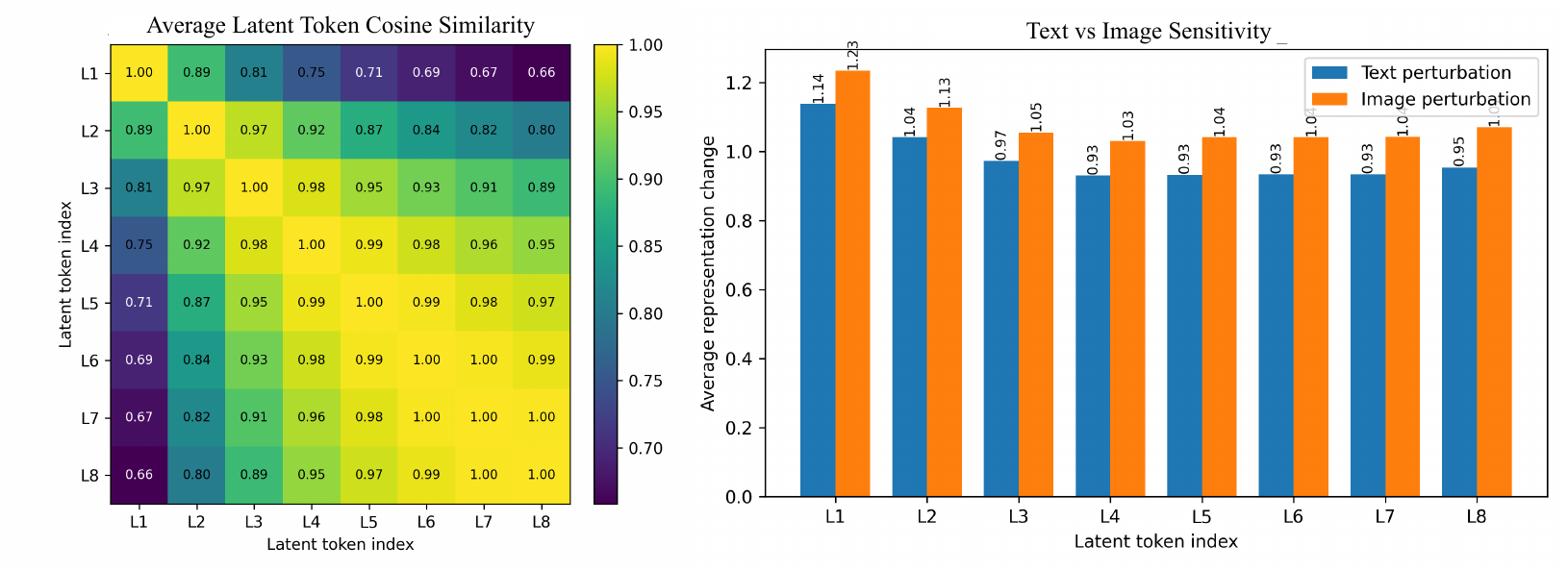}
    \caption{
    Quantitative analysis of latent token properties. 
    \textbf{Left:} Average cosine similarity matrix across latent tokens. 
    \textbf{Right:} Sensitivity of latent representations to image and text permutations.
    }
    \label{fig:sensiti}
\end{figure*}

\section{Out-of-Distribution Generalization on VisualPuzzles}
\label{appendix:ood_visualpuzzles}

To further evaluate whether the learned latent reasoning space generalizes beyond the in-distribution benchmarks used in the main paper, we test our model on VisualPuzzles, an out-of-distribution multimodal reasoning benchmark designed to assess more abstract reasoning skills. VisualPuzzles contains diverse reasoning categories, including algorithmic, analogical, deductive, inductive, and spatial reasoning, and therefore provides a useful testbed for evaluating whether the learned latent representations support transferable reasoning rather than benchmark-specific pattern matching.

Table~\ref{tab:ood} reports the results. Our 7B model achieves the best overall performance (33.90), outperforming all compared baselines, including Qwen2.5-VL-7B~\cite{qwen25}, Deepeyes~\cite{deepeyes}, Pangea-7B~\cite{pangea}, LVR~\cite{lvr}, and LLaVA-OneVision-72B~\cite{Llava-onevision}. Notably, our model obtains the strongest score on deductive reasoning (52.50), with additional gains on analogical and spatial reasoning compared with most competing models. These improvements suggest that the learned latent reasoning states are not merely improving in-domain answer generation, but also provide a more transferable reasoning substrate under distribution shift.

This result complements the main benchmark findings. While the in-distribution results show that AMVL improves multimodal reasoning performance on standard evaluation sets, the VisualPuzzles experiment further indicates that the learned latent space retains useful abstract structure under out-of-distribution conditions. This is consistent with our broader claim that improving train-inference compatibility in latent reasoning can lead not only to better in-domain decoding, but also to stronger reasoning generalization.

\section{Semantic Properties of the Latent Reasoning Space}
\label{app:additional_analysis}

To better understand the learned latent reasoning tokens, we conduct qualitative and quantitative probing analyses, shown in Figure~\ref{fig:heatmap} and Figure~\ref{fig:sensiti}.

\paragraph{Visual grounding.}
Figure~\ref{fig:heatmap} visualizes token-level relevance heatmaps obtained via occlusion-based sensitivity analysis. For each latent token $L_i$, we mask image patches and measure the resulting representation shift, where warmer colors indicate higher sensitivity. Across different queries, the latent tokens consistently respond to task-relevant visual regions rather than diffuse global context. For example, when the query asks about the color of a motorcycle or a comb, the strongest responses localize around the queried objects and their nearby boundaries. The tokens also exhibit diverse but overlapping relevance patterns: some capture broader context, while others focus more sharply on local details. This suggests that the latent block forms a sequence of progressively refined visual abstractions, rather than collapsing into redundant slots. More importantly, the strong localization to query-relevant regions is consistent with our main claim that continuous latent reasoning can preserve fine-grained perceptual grounding without forcing intermediate reasoning into discrete language tokens.

\paragraph{Latent geometry and sensitivity.}
Figure~\ref{fig:sensiti} provides complementary quantitative evidence. The token-wise cosine similarity matrix (left) shows a banded structure: adjacent latent tokens are more similar, while distant tokens (e.g., $L_1$ and $L_8$) are less aligned. This pattern suggests a smooth but non-collapsed latent trajectory. The perturbation analysis (right) further shows that latent representations are more sensitive to image permutations than to text permutations. This indicates that the learned latent space is strongly grounded in visual input, while still remaining conditioned on the textual query. Taken together, these results suggest that AMVL's latent reasoning space is strongly grounded in visual evidence, while still being shaped by the textual query. This is consistent with our broader motivation of alleviating the language-space bottleneck in multimodal reasoning.

\begin{table}[t]
\centering
\caption{Latent-space spread and prior--posterior alignment statistics. Lower spread and paired L2 indicate more compact and better-aligned latent geometry, while higher cosine indicates stronger directional consistency.}
\label{tab:latent_spread_stats}
\resizebox{0.99\textwidth}{!}{
\begin{tabular}{lccccc}
\toprule
Method & Prior Spread $S_p \downarrow$ & Posterior Spread $S_q \downarrow$ & Paired L2 $D_{\text{L2}} \downarrow$ & Cosine $S_{\text{cos}} \uparrow$ & Mean Shift $D_{\text{shift}} \downarrow$ \\
\midrule
NTP  & 1.9243 & 2.2326 & 16.9488 & -0.0190 & 16.6418 \\
NTP + Rev-KL & 0.4453 & 0.5177 & 3.6960  & 0.9594  & 3.6294  \\
NTP + Fwd-KL & 0.6359 & 0.8707 & 5.6528  & 0.8697  & 5.5484  \\
AMVL & 0.6786 & 0.8257 & 5.4317  & 0.8883  & 5.3240  \\
\bottomrule
\end{tabular}
}
\end{table}

\section{Latent Spread Analysis}
\label{appendix:latent_spread}

To better understand how different training objectives shape the latent reasoning space, we analyze the dispersion of prior and posterior latent means under the four main training variants from our ablation study (Table~\ref{tab:loss_ablation}): NTP, NTP + Rev-KL, NTP + Fwd-KL, and AMVL. Here, NTP denotes the baseline trained only with next-token prediction, without explicit latent regularization. NTP + Fwd-KL denotes the forward-KL alignment variant, while NTP + Rev-KL denotes the reverse-KL regularized variant. AMVL denotes the full bidirectional objective. This analysis complements the main ablation study by examining how each objective shapes the geometry of the learned latent space.

For each validation example $n \in \{1, \dots, N\}$, we extract the prior mean $\mu_p^{(n)} \in \mathbb{R}^{k \times d}$ and posterior mean $\mu_q^{(n)} \in \mathbb{R}^{k \times d}$ from the trained model, where $k$ is the number of latent slots and $d$ is the latent dimension. To obtain a sample-level representation, we apply slot-wise mean pooling:
\[
\bar{\mu}_p^{(n)} = \frac{1}{k}\sum_{i=1}^{k}\mu_{p,i}^{(n)}, \qquad \bar{\mu}_q^{(n)} = \frac{1}{k}\sum_{i=1}^{k}\mu_{q,i}^{(n)}.
\]
For each branch (prior or posterior), we compute the global center across the validation set, denoted as $c_p$ and $c_q$:
\[
c_p = \frac{1}{N}\sum_{n=1}^{N}\bar{\mu}_p^{(n)}, \qquad c_q = \frac{1}{N}\sum_{n=1}^{N}\bar{\mu}_q^{(n)}.
\]

We define the average latent spread for the prior ($S_p$) and posterior ($S_q$) as the mean Euclidean distance from each sample to its respective global center:
\[
S_p = \frac{1}{N}\sum_{n=1}^{N} \left\|\bar{\mu}_p^{(n)} - c_p\right\|_2, \qquad S_q = \frac{1}{N}\sum_{n=1}^{N} \left\|\bar{\mu}_q^{(n)} - c_q\right\|_2.
\]
Notably, these spread statistics quantify the dispersion of sample-level pooled latent means across the dataset, rather than the covariance or support width of each individual latent distribution.

\textbf{Geometric Interpretation:} $S_p$ and $S_q$ quantify the global concentration of sample-level latent representations across the dataset. Lower spread indicates that the pooled latent means are more tightly clustered around the global center, whereas higher spread suggests stronger cross-sample drift.

To comprehensively quantify the alignment between the prior and posterior branches, we formulate three paired-geometry metrics, each capturing a distinct aspect of the train-inference mismatch:
\[
D_{\text{L2}} = \frac{1}{N}\sum_{n=1}^{N} \left\|\bar{\mu}_p^{(n)} - \bar{\mu}_q^{(n)}\right\|_2, \quad
S_{\text{cos}} = \frac{1}{N}\sum_{n=1}^{N} \frac{\langle \bar{\mu}_p^{(n)}, \bar{\mu}_q^{(n)} \rangle}{\|\bar{\mu}_p^{(n)}\|_2 \|\bar{\mu}_q^{(n)}\|_2}, \quad
D_{\text{shift}} = \left\|c_p - c_q\right\|_2.
\]
\begin{itemize}
    \item \textbf{Paired L2 Distance ($D_{\text{L2}}$)} measures the \emph{absolute instance-level error}. It evaluates how far apart the target-agnostic prior and the target-aware posterior are for the exact same input sample.
    \item \textbf{Cosine Similarity ($S_{\text{cos}}$)} measures the \emph{directional consistency}. Regardless of magnitude, it evaluates whether the prior and posterior point toward the same semantic region in the high-dimensional space.
    \item \textbf{Mean Shift ($D_{\text{shift}}$)} measures the \emph{systematic global bias}. It captures the macro-level translation between the entire prior distribution and the posterior distribution, indicating overall domain drift.
\end{itemize}

Table~\ref{tab:latent_spread_stats} summarizes these statistics. Several patterns are broadly consistent with the benchmark trends reported in the main text.

\paragraph{NTP leads to severe latent mismatch and the weakest overall reasoning behavior.}
Without explicit latent regularization, NTP produces the most dispersed latent geometry for both branches ($S_p = 1.9243, S_q = 2.2326$), together with the worst prior--posterior alignment ($D_{\text{L2}} = 16.9488, S_{\text{cos}} = -0.0190, D_{\text{shift}} = 16.6418$). This suggests that the posterior drifts freely toward target-dependent latent regions while the prior remains poorly calibrated for inference-time usage. Such severe train--inference mismatch is consistent with the weaker benchmark performance of NTP-only observed in the ablation study.

\paragraph{NTP + Rev-KL produces the most concentrated sample-level latent geometry, but not the strongest downstream performance.}
Among all variants, NTP + Rev-KL yields the smallest sample-level spread and the strongest prior--posterior alignment under the considered metrics ($S_p = 0.4453, S_q = 0.5177, D_{\text{L2}} = 3.6960, S_{\text{cos}} = 0.9594$). This suggests that reverse-side regularization is effective at reducing cross-sample drift of latent means and improving compatibility with the learned prior. Importantly, this concentration is measured at the level of pooled latent means across validation examples, and does not imply that each individual posterior becomes sharper. In fact, reverse KL can still encourage broader support within a sample while making the sample-level latent geometry more globally concentrated. However, NTP + Rev-KL does not achieve the best downstream performance, implying that stronger geometric concentration alone is not sufficient: overly strong posterior regularization may reduce useful target-conditioned variability needed for reasoning.

\paragraph{NTP + Fwd-KL substantially improves over NTP, but still exhibits an imbalanced latent geometry.}
The forward-KL variant (NTP + Fwd-KL) resolves most of the catastrophic mismatch observed under NTP, improving both paired distance and cosine alignment by a large margin. At the same time, the prior branch remains more concentrated across examples ($S_p = 0.6359$), while the posterior branch is still more dispersed across examples ($S_q = 0.8707$), yielding a relatively large prior--posterior spread gap ($\Delta S = S_q - S_p = 0.2348$). This pattern is consistent with a one-sided calibration regime in which the prior is trained to chase the posterior, while the posterior itself remains more weakly constrained and therefore relatively target-dependent. Correspondingly, NTP + Fwd-KL improves benchmark accuracy over NTP, but remains below the full AMVL objective.

\paragraph{AMVL yields a more favorable trade-off between prior expressiveness and posterior regularization.}
Compared with NTP + Fwd-KL, AMVL slightly increases the spread of sample-level prior means ($S_p = 0.6786$) while slightly reducing the spread of sample-level posterior means ($S_q = 0.8257$), thereby significantly shrinking the prior--posterior spread gap ($\Delta S = 0.1471$). It also improves paired L2 distance, cosine similarity, and mean shift over the forward-only baseline. Importantly, although AMVL does not produce the most concentrated latent geometry overall, it achieves the strongest downstream benchmark performance. This suggests that the most effective latent reasoning space is not the one with the tightest global concentration, but the one that best balances inference-time prior expressiveness with training-time posterior regularization. This is consistent with the main design of AMVL, which combines prior alignment with posterior control to improve train-inference compatibility.

\section{Additional Ablation Studies}
\label{app:more_ablation}
In this section, we provide extended empirical analyses to further validate the structural design, optimization stability, and inference robustness of our AMVL framework. Specifically, we investigate: (1) the impact of different \textbf{variational head architectures}, confirming the necessity and efficiency of our lightweight LLM-native design; (2) the \textbf{sensitivity to loss weights}, demonstrating that AMVL's dual-KL objective remains robust across a diverse range of hyperparameter settings; and (3) the \textbf{robustness to inference-time latent sampling} under various temperatures, which empirically verifies the smoothness and stability of the learned continuous latent space.

\begin{table}[t]
\caption{Ablation study of variational head architectures. We compare a linear head, a standard MLP head, our LLM-native head, and a deeper variant of the same design. 
}
\label{tab:head_ablation}
\centering
\begin{tabular}{lccc}
\toprule
Method & V$^*$ & HRBench4K & HRBench8K \\
\midrule
Linear head & 76.96 & 70.00 & 66.25 \\
MLP head & 75.92 & 69.88 & 65.50 \\
LLM-native head & 84.29 & 72.12 & 68.50 \\
Deeper LLM-native head & 75.39 & 70.25 & 65.50 \\
\bottomrule
\end{tabular}
\vspace{-7pt}
\end{table}

\paragraph{Effect of the variational head architecture.}
Table~\ref{tab:head_ablation} compares different variational head architectures, including a linear head, a standard MLP head, our proposed LLM-native head, and a deeper variant of the same design. The proposed lightweight LLM-native head performs best across benchmarks, indicating that variational parameterization benefits from being aligned with the architectural patterns of the underlying MLLM rather than relying on a separately designed generic projection head.

Interestingly, the deeper variant performs worse than the default head. We hypothesize that, in our setting, the variational head mainly serves as a lightweight readout of latent statistics from the shared MLLM hidden states, rather than as an independent high-capacity encoder. Increasing the head depth may therefore weaken feature-space alignment with the backbone, complicate prior-posterior calibration, and make the posterior more prone to fitting target-conditioned shortcut information.

\begin{table}[t]
\caption{Sensitivity analysis of loss weights. We vary the coefficients of the next-token prediction loss, forward KL alignment, and reverse KL regularization to evaluate the robustness of AMVL to objective weighting.}
\label{tab:weight_sensitivity}
\centering
\begin{tabular}{cccccc}
\toprule
$\lambda_{\mathrm{NTP}}$ & $\beta$ & $\gamma$ & V$^*$ & HRBench4K & HRBench8K \\
\midrule
1 & 1   & 1   & 84.29 & 72.12 & 68.50 \\
1 & 0.5 & 1   & 81.68 & 71.62 & 67.75 \\
1 & 1   & 0.5 & 80.63 & 72.12 & 68.12 \\
2 & 1   & 1   & 80.63 & 73.00 & 68.88 \\
2 & 0.5 & 1   & 81.68 & 72.00 & 69.00 \\
2 & 1   & 0.5 & 81.68 & 70.88 & 67.88 \\
\bottomrule
\end{tabular}
\vspace{-7pt}
\end{table}

\paragraph{Sensitivity to loss weights.}
Table~\ref{tab:weight_sensitivity} reports a sensitivity analysis over the objective weights. Overall, AMVL remains reasonably stable across a range of coefficient settings, indicating that the gains do not depend on a narrowly tuned loss balance. At the same time, changing the relative strengths of the NTP, forward KL, and reverse KL terms affects the trade-off across benchmarks, which is consistent with their different roles in latent-space learning.

In particular, the forward KL term mainly improves prior calibration for inference, while the reverse KL term controls posterior sharpness and support compatibility. Over-emphasizing either side can hurt performance: excessively strong reconstruction pressure may weaken latent regularization, whereas overly strong reverse regularization may suppress useful target-conditioned information. Overall, the results support the design of AMVL as a balanced objective that jointly improves prior calibration and posterior regularization.

\begin{table}[t]
\caption{Ablation study of the stop-gradient design in our variational alignment objectives. We evaluate the effect of removing stop-gradient from the prior-alignment term, the posterior-regularization term, or both.}
\label{tab:sg_ablation}
\centering
\begin{tabular}{lcccc}
\toprule
Method & V$^*$ & HRBench4K & HRBench8K &VisualPuzzles\\
\midrule
w/o sg in Prior Alignment & 83.25 &72.12  & 67.62&32.53 \\
w/o sg in Posterior Regularization & 80.10&72.25 & 66.75 & 32.79 \\
w/o sg in Both & 81.68 & 72.50 & 67.75 &33.05\\
Full & 84.29 & 72.12 & 68.50&33.90 \\

\bottomrule
\end{tabular}
\end{table}

\paragraph{Effect of stop-gradient design.}
As shown in Table~\ref{tab:sg_ablation}, the full model performs best overall, demonstrating the importance of the stop-gradient design in our variational alignment objectives. Removing stop-gradient from either the prior-alignment term or the posterior-regularization term leads to clear performance drops, especially on V$^*$ and HRBench8K. The degradation is most pronounced when the prior-alignment term is allowed to update the posterior, suggesting that this term works best when the posterior serves as a fixed, answer-informed teacher for the prior. Similarly, removing stop-gradient from posterior regularization weakens the intended constraint on the posterior by allowing the prior to co-adapt.

Additionally, removing stop-gradient from both terms does not cause catastrophic failure, but still underperforms the full model across most benchmarks. This indicates that the gain of our method comes not only from the presence of bidirectional KL objectives, but more importantly from the decoupled gradient flow they enforce. Overall, the results validate our design principle that prior alignment and posterior regularization should update only their intended target distributions.

\begin{table}[t]
\centering
\caption{Robustness to inference-time latent sampling under different temperatures. Higher is better.}
\label{tab:temperature_ablation}
\begin{tabular}{lccccc}
\toprule
\multirow{2}{*}{Benchmark} & \multicolumn{5}{c}{Inference Temperature ($\tau$)} \\
\cmidrule(lr){2-6}
& $\tau = 0.0$ & $\tau = 0.2$ & $\tau = 0.5$ & $\tau = 0.8$ & $\tau = 1.0$ \\
\midrule
V$^*$          & 84.29 & 83.25 & 83.25 & 82.72 & 82.20 \\
HRBench4K      & 72.12 & 71.88 & 71.75 & 71.62 & 71.00 \\
HRBench8K      & 68.50 & 67.88 & 68.00 & 67.50 & 67.62 \\
VisualPuzzles  & 33.90 & 34.50 & 33.39 & 33.56 & 32.71 \\
\bottomrule
\end{tabular}
\end{table}

\paragraph{Robustness to Inference-Time Latent Sampling.}
\label{appendix:temperature_ablation}
To further test the stability of the learned latent reasoning space at inference time, we evaluate AMVL under stochastic prior sampling with different temperatures. In the main experiments, inference uses the prior latent mean directly. Here, we instead inject Gaussian noise scaled by the prior standard deviation and a temperature parameter $\tau$:
\[
z = \mu + \tau \cdot \epsilon \odot \sigma,\qquad
\epsilon \sim \mathcal{N}(0, I), \quad \sigma = \exp\!\left(\tfrac{1}{2}\log \mathrm{var}\right).
\]
When $\tau=0$, this reduces to deterministic mean-based inference. As $\tau$ increases, the sampled latent variables become increasingly stochastic. At inference time, both $\mu$ and $\log \mathrm{var}$ are obtained from the prior branch, so the injected noise reflects uncertainty under the learned inference-time latent distribution rather than any target-aware posterior information.

We evaluate the trained model under $\tau \in \{0.0, 0.2, 0.5, 0.8, 1.0\}$ on V$^*$, HRBench4K, HRBench8K, and VisualPuzzles. The results are reported in Table~\ref{tab:temperature_ablation}.

The results show that AMVL remains reasonably stable under moderate latent perturbations. On V$^*$, performance drops from 84.29 at $\tau=0.0$ to 83.25 at $\tau=0.2$ and remains at the same level at $\tau=0.5$, indicating that mild stochasticity does not substantially harm inference. HRBench4K exhibits a similarly gradual degradation, decreasing from 72.12 to 71.00 as $\tau$ increases from 0.0 to 1.0. HRBench8K is even more stable, with only minor fluctuations across temperatures. On VisualPuzzles, moderate stochastic sampling at $\tau=0.2$ slightly improves the score (34.50 vs.\ 33.90), while higher temperatures eventually reduce performance.

Overall, these results suggest that the latent reasoning space learned by AMVL is not overly brittle to local perturbations at inference time. Moderate sampling noise only leads to small performance changes, and in some cases can even slightly improve generalization, which is consistent with the learned latent space being locally smooth rather than highly unstable. At the same time, performance gradually degrades as $\tau$ becomes large, indicating that excessive stochasticity still moves the latent state away from the most reliable inference region. Together with the latent-space analyses in Appendix~\ref{appendix:latent_spread}, this sampling ablation suggests that improved prior-posterior calibration in AMVL not only benefits deterministic inference, but also makes the learned latent reasoning space more robust to moderate stochastic perturbations.

\section{Limitations and Future Work} 
While extended ablations (Appendix~\ref{app:more_ablation}) confirm AMVL's stability and efficacy, our current empirical validation is limited to the 7B parameter scale. A critical direction for future work is scaling AMVL to larger foundational models (e.g., 70B+) to investigate whether massive parameter counts induce the spontaneous emergence of more complex, generalized latent reasoning structures, further advancing the frontier of continuous multimodal reasoning.

\section{Broader Impacts} 
Our proposed latent thinking model aims to enhance the reasoning efficiency and interpretability of large language models. Positively, this could lower computational costs for complex reasoning tasks and make model decision-making processes more transparent to users. However, we also acknowledge potential negative societal impacts. Enhanced reasoning capabilities could potentially be misused to generate more sophisticated disinformation or automate malicious activities such as phishing. 

\newpage

\end{document}